\documentclass{article}
\usepackage{arxiv}
\usepackage[numbers, sort, comma, square]{natbib}

\usepackage[utf8]{inputenc}
\usepackage{tikzsymbols}
\usepackage{microtype}
\usepackage{graphicx}
\usepackage{subcaption}
\usepackage{booktabs} 
\usepackage{appendix}
\usepackage[hidelinks,colorlinks,allcolors=blue]{hyperref}       
\usepackage{enumitem}
\usepackage{algorithm,algorithmic}

\usepackage{color}
\usepackage{amssymb, amsthm}
\usepackage{amsfonts}
\usepackage{amsmath,centernot}
\usepackage{calrsfs}
\usepackage{url}
\usepackage{footmisc}
\usepackage{tikz}

\newtheorem{theorem}{Theorem}
\newtheorem{lemma}{Lemma}

\newtheorem{assumption}{Assumption}
\usepackage{epstopdf}
\usepackage{wrapfig}
\usepackage{float}
\usepackage{cleveref}
\usepackage{tikz-cd}
\usetikzlibrary{matrix, calc, arrows, automata, positioning}
\usepackage{mathabx}

\newcommand{\set}[1]{\{#1\}}			
\newcommand{\B}{\mathcal{B}}
\newcommand{\R}{\mathbb{R}}                     

\newcommand{\E}[1]{\mathrm{E}\left[ #1 \right]}
\newcommand{\Em}[2]{\mathrm{E}_{#1} \left[ #2 \right]}
\newcommand{\var}[1]{\mathrm{Var}\left( #1 \right)}
\newcommand{\Pm}{\mathrm{Pr}}         
\newcommand{\Pro}[1]{\Pm \left[#1\right]}

\newcommand{\abs}[1]{\left| #1 \right|}

\newcommand{\bigo}[1]{O\left( #1 \right)}

\newcommand{\norm}[1]{\lVert #1 \rVert}

\newcommand{\eps}{\epsilon}
\renewcommand{\epsilon}{\varepsilon}
\renewcommand{\le}{\leqslant}
\renewcommand{\leq}{\leqslant}
\renewcommand{\ge}{\geqslant}
\renewcommand{\geq}{\geqslant}
\newcommand{\rank}{\mathrm{rank}}

\newcommand{\T}{\mathbf{T}}
\newcommand{\bS}{\mathbf{S}}

\newcommand{\obs}{\textrm{obs}}

\newcommand{\Hy}{\mathcal{H}}
\newcommand{\opt}{\textrm{OPT}}

\newcommand{\nnz}{\mathit{nnz}}
\newcommand{\poly}{\textrm{poly}}
\newcommand{\sw}{\textrm{sw}}
\newcommand{\cov}{\textrm{cov}}

\newcommand{\atet}{\textrm{ATET}}

\newcommand{\CI}{\mathrel{\perp\mspace{-10mu}\perp}}

\newcount\Comments  
\newcommand{\kibitz}[2]{\ifnum\Comments=1{\color{#1}{#2}}\fi}

\newcommand{\hatT}{\widehat{\T}}

\newcommand{\checkT}{\widecheck{\T}}
\newcommand{\tildeT}{\widetilde{\T}}

\newcommand{\splitatcommas}[1]{%
	\begingroup
	\begingroup\lccode`~=`, \lowercase{\endgroup
		\edef~{\mathchar\the\mathcode`, \penalty0 \noexpand\hspace{0pt plus 1em}}%
	}\mathcode`,="8000 #1%
	\endgroup
}

\newcommand{\tA}{\tilde{A}}
\newcommand{\tL}{\tilde{L}}
\newcommand{\tY}{\tilde{Y}}

\allowdisplaybreaks

\usetikzlibrary{decorations.shapes}
\tikzset{decorate sep/.style 2 args=
	{decorate,decoration={shape backgrounds,shape=circle,shape size=#1,shape sep=#2}}}
\title{Weighted Tensor Completion for Time-Series Causal Inference}
\author{Debmalya Mandal \\
Data Science Institute, \\
Columbia University \\
\texttt{dm3557@columbia.edu}\\
\And
David C. Parkes\\
Paulson School of Engineering and Applied Sciences,\\
Harvard University\\
\texttt{parkes@eecs.harvard.edu}\\
}

\begin{document}
\maketitle

\begin{abstract}
 	Marginal Structural Models  (MSM) are the
	most popular models for causal inference from time-series
	observational data. However, they have two main drawbacks:
	(a) they do not capture subject heterogeneity, and (b) they
	only consider fixed time intervals and do not scale gracefully with longer intervals.
	In this work, we propose a new family of MSMs to
	address these two concerns. We model the potential outcomes
	as a three-dimensional tensor of low rank, where the three
	dimensions correspond to the agents, time periods and the
	set of possible histories. Unlike the traditional MSM, we
	allow the dimensions of the tensor to increase with the
	number of agents and time periods. We set up a weighted
	tensor completion problem as our estimation procedure, and
	show that the solution to this problem converges to the true
	model in an appropriate sense. Then we show how to
	solve the estimation problem, providing
	conditions under which we can 
	approximately and efficiently solve the
	estimation problem.  Finally we propose an algorithm based on projected
	gradient descent, which is easy to implement, and evaluate its performance on a simulated dataset.
\end{abstract}

\section{Introduction}
The main challenge in causal inference is the estimation of  a
causal quantity of interest from observational data. Often such datasets involve
individuals who are subject to treatments over multiple time periods, and 
we want to estimate the effect of a policy on
the outcome. 
For example, consider a ride-sharing company,  which
records several variables such as the number of trips, and trip
origins and destinations, for each rider, and based on this
information decides whether or not to provide monthly discounts.
After running this experiment for several months, the company is
interested to know whether providing discounts increases the number of
trips taken. If the answer is yes, the company might also want to find
a policy that would further increase the number of trips taken.

A second example comes from~\citet{ANRR14}, who
consider a fundamental problem in political science: does
democracy cause economic development, in relation to
autocracy? The authors collect data from 184 countries over
more than half a century, including  GDP per capita, current 
policital situation (democracy or autocracy), net financial inflow etc.The
goal is to find out whether democracy increases GDP of the
countries over the periods when the country was under
democracy. 

%
The main question underlying the two examples is the following: what is the effect of a treatment
policy over the subjects who are assigned the treatment?  This
quantity is known as the {\em average treatment effect over the
	treated} (ATET).
%
%
%
The main challenge in estimating the effect of time-varying treatments on the outcomes
is the presence of time-varying confounders. These are the variables that affect both the outcomes
and time-varying treatments.  


In a seminal work, \citet{Robins00} proposed {\em Marginal Structural Models} to model
the potential outcomes under time-varying treatments and showed how to remove the bias
due to the presence of time-varying confounders. Even though
{\em Marginal Structural Models} (MSMs)~\cite{Robins00} are widely
used to perform causal inference under  time-varying treatments, they have two  main
drawbacks: (a) they do not capture subject heterogeneity, and (b)
they
only consider fixed time intervals and do not scale gracefully with longer intervals.
This latter
limitation comes about because the number of parameters scales
linearly with the length of the time interval, and with a fixed number of
agents there is not enough data to estimate the parameters of the model.
For example, the effect of ridesharing discounts will vary
by different communities of riders, and may only be realized over a
long period of time. 
%

In this work, we propose a new form of MSM to address these 
drawbacks. We assume that the potential outcomes are generated from a
three-dimensional tensor of low rank, where the dimensions correspond
to the agents, time intervals, and set of possible histories. Intuitively, the rank
of the  tensor can be interpreted as a measure of the heterogeneity of the agents or the
time periods. For example, if the rank is $r$, then each agent can be described as some combination of $r$ 
underlying groups.  We assume the rank of the tensor is low, but
we allow the dimensions of the tensor to increase with the
number of agents and time periods. 

{\bf Contributions}: In order to estimate the outcome model, we set up
a weighted tensor completion problem, and show that the solution
converges to the true model. Compared to the traditional MSMs~\citet{Robins00}, we prove
convergence for two cases -- when the number of agents $N$ is fixed
and the length of the time interval $T$ increases and when $T$ is
fixed and $N$ increases. In particular, if the outcome at every time
period depends only on the  history of length $k$, then as
long as $k$ is bounded by logarithm of the increasing variable (be it
$N$ or $T$), our method guarantees convergence.  We solve the weighted
tensor completion in two steps. First, we convert it to a weighted
tensor approximation problem with an additive loss, where the loss
goes to zero as either $N$ or $T$ increases. Then we turn to solving
this weighted low-rank approximation problem, and provide conditions
under which we can approximately solve the estimation problem in
polynomial time. To the best of our knowledge, ours is the first
additive approximation algorithm for the noisy weighted tensor completion 
that runs in polynomial time under reasonable conditions. Finally, we propose an algorithm based on projected
gradient descent, which is easy to implement, and show that on a simulated dataset,
it performs better than the classical marginal structural models. Additionally, we also perform sensitivity analysis of our algorithm for various values of the assumed parameters.

\subsection{Related Work}


The fundamental problem of causal inference is that for 
each unit we  observe only one of  two possible outcomes-- either the
outcome corresponding to the  treatment or the outcome corresponding to the control, but not both. 
A standard approach
is to 
use the Neyman-Rubin potential outcomes framework~\cite{Rubin74}, where
for each unit and each intervention ($0$ or
$1$), there are two potential outcomes $Y_0$ and $Y_1$, and we only
observe one of these two outcomes. The traditional
focus has been on estimating  the {\em average
	treatment effect} (ATE), which measures the difference
in average outcomes under treatment than without treatment. 
%
However, with
ever-increasing data and  improvements in machine learning
algorithms, several recent papers have devised algorithms to discover
heterogeneous treatment effects. They often involve machine learning
techniques such as Bayesian nonparametrics \cite{Hill11}, random forests
\cite{WA18,AI16}, and deep learning \cite{SJS16,JSS16,YJS18}. Although we will be working
with the potential outcomes framework, there has also been siginificant effort in using
structural causal models as a framework for causality \cite{Pearl18},
including attention to heterogeneous effects \cite{SP12, Pearl17}.
Although there have been several attempts \cite{PJS13} to generalize these structural causal models 
for to consider multi-variate time-series data, 
 we are not aware of any work on combining these methods
with the kinds of temporal settings studied here. 

Epidemiologists and biostatisticians have 
considered the problem of estimating the causal effect of a policy
that
applies treatments over multiple time periods. \citet{Robins86} proposed the {\em marginal structural model} (MSM), as 
a way to measure the causal effect of a time-varying treatment in the presence of time-varying confounders.
Suppose, for example,
that a policy applies a binary treatment over  $T$ time periods.  MSM models each of the $2^T$ potential
outcomes through a parametric model with parameter $\beta$. \citet{Robins86} further showed that the solution to a maximum
weighted likelihood correctly estimates the quantity $\beta$. MSM has been adopted in various domains to estimate
the causal effect in a longitudinal study. Examples include  the effect of different drugs on the HIV patients \cite{RHB00},
the effect of loneliness on depression \cite{VHLT+11}, finance \cite{BS19}, and political science \cite{BG18}.

	\begin{figure*}[t!]
		\centering
		\begin{tikzpicture}
		\node[state] at (0,0) (l1) {$L_{i,1}$};
		\node[state] at (2,0) (a1) {$A_{i,1}$};
		\node[state] at (4,0) (y1) {$Y_{i,1}$};
		\node[state] at (6,0) (l2) {$L_{i,2}$};
		\node[state] at (8,0) (a2) {$A_{i,2}$};
		\node[state] at (10,0) (y2) {$Y_{i,2}$};
		\draw[->] (l1) edge (a1);
		\draw[->] (a1) edge (y1);
		\draw[->] (l1) edge[bend left] (y1);
		\draw[->] (y1) edge (l2);
		\draw[->] (a1) edge[bend right] (l2);
		\draw[->] (l1) edge[bend right] (l2);
		\draw[->] (a1) edge[bend left] (a2);
		\draw[->] (l1) edge[bend left] (a2);
		\draw[->](l2) edge (a2);
		\draw[->] (a2) edge (y2);
		\draw[->] (a1) edge[bend right] (y2);
		\draw[->] (l1) edge[bend left] (y2);
		\draw[->] (l2) edge[bend right] (y2);
		\draw[->] (y1) edge[bend right] (a2);
		\draw[decorate sep={1mm}{3mm},fill] (11,0) -- (12,0);
		\end{tikzpicture}
		\caption{A directed acyclic graph describing the model for individual $i$. Since $L_{i,t}$'s affect both treatments and outcomes, they are time-varying confounders.}
		\label{fig:causal_graph}
	\end{figure*}
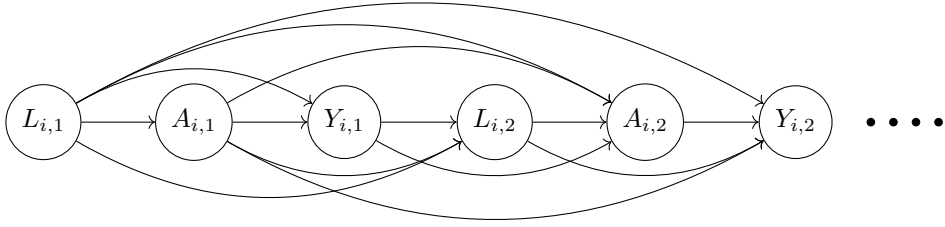

There have been very few attempts to generalize these models to capture important aspects such as 
heterogeneous effects, or large numbers of time-periods. Bayesian non-parametric methods have been used to estimate effect of time-varying interventions \cite{SSS17, XXS16}. They use gaussian process to model the progression of time-series, and can also estimate the effects of continuously varying treatments. However, these methods often make strong assumptions and  do not consider subject level heterogeneity and are often. Moreover, inference is often complicated with Bayesian methods, and the methods do not scale well with $N$ or $T$. On the other hand, \citet{LAS18} recently introduced recurrent marginal structural model, which is a recurrent neural network based architecture to forecast outcome in the future. Even though this model is an interesting generalization of classical MSM, it still considers a homogeneous MSM, and it needs a large number of policy evaluations to train the network.

\citet{NVMM+07} define a history-adjusted MSM,
which considers potential outcomes dependent on a short history instead of the full history of length $T$. 
Similar to~\citet{Robins86},
they propose an estimator based on maximum weighted likelihood, but
that fails to 
capture heterogeneous effects over the population.
The most closely related prior work is that of~\citet{ABDI+18}, who use matrix completion methods to estimate average treatment effects and other related causal quantities for the  time-varying treatment setting.
They model the potential outcomes using a matrix of low rank and provide an estimator. 
However, they do not consider the effect of past treatments on the outcomes. Rather,  
the potential outcome at each  time step depends only on
the current treatment. 
%
%
\citet{BAWM18} do consider time-varying treatments, but model treatment effect conditioned on a given history and 
under the same underlying policy.
Since they prefer not
to  directly model the environment,
their method cannot be used to estimate
the average treatment effect or other related quantities under a different policy. 


Finally, we use tensors to model the potential outcomes,
and in recent years, there have been several applications of tensor
methods in various machine learning problems~\cite{AFHKL12}. Our main optimization
problem is weighted tensor completion problem, which tries to estimate
the missing entries of a tensor from the observed entries. Tensor
completion is well-studies  \cite{BM16,YZ16,MS18}, but the problem of weighted tensor
completion is relatively unexplored. We convert the weighted tensor
completion problem into a weighted tensor approximation problem. 
This problem is intractable in general, but under suitable conditions,
\citet{SWZ19} recently developed an efficient algorithm.

\section{Model}

For $t=1,\ldots,T$, $A_{i,t}$ denotes the treatment assigned to subject $i$ at time $t$, and $L_{i,t}$ denotes the observed time-varying covariate at time $t$.  
For $t = 1,\ldots,T$, $Y_{i,t}$ denotes the observed outcome for unit $i$ at time $t$ and depends on the history of the treatments
assigned to agent $i$ at time $t$, and also on the sequence of time-varying covariates of agent $i$. We use the following notation for a sequence of
treatments. $A_{i,t':t''}$ denotes the sequence of treatments from $t'$ to $t''$ i.e. $A_{i,t'}, A_{i,t'+1},\ldots, A_{i,t''}$. A sequence of covariates, 
$L_{i,t':t''}$ is defined analogously. We will use lowercase variables to denote interventions of the random variables, e.g., $a_{i,t}$ denotes
a fixed assignment of $A_{i,t}$. 
The same notation applies to co-variates, and the outcomes.

The directed acyclic graph (figure~\ref{fig:causal_graph}) 
represents the relationship among different variables. For each $i$ and $t$, a {\em policy} determines $A_{i,t}$, i.e., the treatment assigned to individual $i$ at time $t$.
In general, such a policy can be dynamic, so that the
action $A_{i,t}$ depends  on the  history up to time $t$. In such a case, we will write
$\Pro{A_{i,t} | A_{i,1:t-1}, L_{i,1:t} }$ for the probability assigned to the treatment $A_{i,t}$ given  past treatment sequence of length $t-1$, $A_{i,1:t-1}$, and the realization of the past co-variate
sequence of length $t$, $L_{i,1:t}$.\footnote{We assume the policy is known i.e. the conditional probabilities of the treatment assignments
	are known. We leave the problem of estimating these probabilities from the data as  future work. 
} Note that, we omit the past outcomes $Y_{i,1:t-1}$ in determining the treatment at time $t$ as there is no direct edge from $Y_{i,t'}$ to $Y_{i,t}$ for $t' < t$. But this is without loss of generality, as $Y_{i,t'}$ can be included in $L_{i,t'+1}$.

The covariates $L_{i,t}$ are time-varying and can depend
the entire history up to time $t$. 
%
In full generality, the outcome at any time might also depend on the entire
treatment history, but we make the following assumption about the
outcome for any individual, say $i$.
\begin{assumption} \label{as:outcome}
	The outcome at time $t$, $Y_{i,t}$ depends only on the past treatment history of length $k$, $A_{i,t-k+1:t}$.
\end{assumption}
However,
the outcome can depend on the entire sequence of time-varying confounders at time $t$, $L_{i,1:t}$. Assumption \ref{as:outcome} helps us index the potential outcomes
at each time by length $k$ histories as $Y_{i,t}(a_{i,t-k+1:t}, \ell_{i,1:t})$. We are interested in the marginal outcomes
$\E{Y_{i,t}(a_{i,t-k+1:t})}$, where the effect of the time-varying covariates have been marignalized. This corresponds to the outcome $Y_{i,t}$ if we intervene on the nodes $A_{i,t-k+1:t}$, set them to the value $a_{i,t-k+1:t}$, remove edges incoming to $A_{i,t-k+1:t}$ and leave the rest of the graph in figure \ref{fig:causal_graph} unchanged.

%

\subsection{Outcome Model}

The marginalized potential outcomes  for an individual $i$ at time $t$, written as $\E{Y_{i,t}(a_{i,t-k+1:t})}$, are indexed by the past treatment history of length $k$. This implies that there are $N \times T \times 2^k$ potential outcomes out of which we observe realizations of  $N \times T$
potential outcomes.\footnote{In some scenarios,  potential outcomes can exhibit structure, e.g. if a subject's response
	at time $t$ depends only on how many times she was given the treatment in the last $k$ rounds. This implies, for each $i$ and $t$, there are only $k+1$ distinct potential
	outcomes. Our algorithm need
	not be aware of such a structure, and the results are stated without this requirement. Introducing this assumption would
	only lead to improved, positive results. }
We now introduce the outcome model. There is a tensor $\T^\star_{N,T}$ of dimension of $N \times T \times 2^k$, such that the outcome for subject 
$i$ at time $t$ for a $k$ length history $a_{i,t-k+1:t}$ is given as:
\begin{equation}\label{eq:outcome-model}
\E{Y_{i,t}(a_{i,t-k+1:t})} = \T^\star_{N,T}[i,t,a_{i,t-k+1:t}]
\end{equation}
Notice that we use $\T^\star_{N,T}$ to denote the true tensor, as opposed to a fixed $\T$. This is because the underlying model changes as either the number 
of agents $N$ or time periods $T$ increases.
 \Cref{eq:outcome-model} says that the marginal potential outcomes are indexed
by the subject $i$, time period $t$, and any possible treatment history of length $k$, $a_{i,t-k+1:t}$. The variable $k$ controls the dependence of the outcome on past sequence of treatments.
In general, $k$ can be arbitrarily long. However, we need to assume that  $k$ is bounded from above by the logarithm of the larger of $N$ and $T$ in order to estimate the potential outcomes.
Otherwise, the number of missing outcomes grows at a rate larger than the number of observed outcomes, and information-theoretically it is impossible to estimate all the missing outcomes.
\footnote{This is reasonable for the ride-sharing example as the number of trips taken by a rider will depend on his coupons for the past couple of months, but
	not on whether she received coupons several years back.}

\subsection{Sequentially Randomized Experiment}
Since we aim to estimate the marginal potential outcomes from observational data, certain identifying assumptions need to be satisfied. 
Below we state the required assumptions for time-varying treatments, which are generalization of standard assumptions in the literature
on causal inference.
%
 Let
$A_{i,t}^{\obs}$ denote the observed outcome, and $A_{i,t}(\cdot)$
denote the corresponding random variable dependent on the history. The
same notation holds for the outcomes and the covariates.  We define
the following properties:


1. {\em Consistency}: The observed data $\splitatcommas{ (L_{i,1},A_{i,1},Y_{i,1}, L_{i,2},A_{i,2},Y_{i,2},\ldots) }$ is equal to the potential outcomes as follows. For
every history $H_{i,t} = (L_{i,1:t},A_{i,1:t},A_{i,1:t-1})$, we have $Y_{i,t}^{\obs} = Y_{i,t}(H_{i,t}) = Y_{i,t}(A_{i,t-k+1:t}, L_{i,1:t})$, $L_{i,t+1}^{\obs} = L_{i,t+1}(H_{i,t})$,
and $A_{i,t+1}^{\obs} = A_{i,t+1}(A_{i,1:t}, L_{i,1:t})$.

2. {\em Sequential Ignorability}: For each $t$, the potential outcomes are independent of the treatment conditioned on the history at time $t$, i.e., 
$Y_{i,t} \CI A_{i,t} \mid A_{i,1:t}, L_{i,1:t}$.

3. {\em Positivity}: There exists a $\delta > 0$ such that for each $A_{i,1:t-1}, L_{i,1:t}, Y_{i,1:t-1}$, we have
\[
\delta < \Pro{A_{i,t} | A_{i,1:t-1}, L_{i,1:t}} < 1 - \delta
\]

Consistency maps the observed outcomes to the potential outcomes. In
particular, the outcome observed at time $t$, $Y_{i,t}^{\obs}$ is
completely determined by the past treatment history of length $k$, and the time-varying covariates.
If $A_{i,t}$ is chosen based on the history up to time $t$, then
sequential ignorability automatically holds~\cite{BAWM18}. On the
other hand, in an observational study, we must assume there are no
unmeasured confounders for sequential ignorability to hold.

\subsection{Quantities to Estimate}

The literature on causal inference has proposed various quantities to estimate in a setting with time-varying treatments. In the introduction, we talked about
the average treatment effect over the treated (ATET). For a fixed policy and any given assignment $\{A_{i,t}\}_{i,t}$ we define ATET to be the average effect of the treatment over the units that actually received the treatment
under $\{A_{i,t}\}_{i,t}$. Let $S_1 = \set{(i,t) : A_{i,t} = 1 }$ be the set of $(i,t)$ indices under treatment. Then,
\begin{align*}
\atet =\frac{1}{\abs{S_1} } \sum_{(i,t) \in S_1} \E{Y_{i,t}(A_{i,t-k+1:t})} - \E{Y_{i,t}(A_{i,t-k+1:t-1},0)}
\end{align*}
According to the outcome model  in~\eqref{eq:outcome-model}, this becomes
\begin{align}
\atet = \frac{1}{\abs{S_1} } \sum_{(i,t) \in S_1}\T^\star_{N,T}[i,t,A_{i,t-k+1:t}] - \T^\star_{N,T}[i,t,(A_{i,t-k+1:t-1},0)],\label{eq:tensor-atet}
\end{align}
and can be computed easily once we have an estimate of the tensor $\T^\star_{N,T}$. More general statistical estimands like the average treatment
effect of switching from one history $h_1$ to another history $h_2$ of length at most $k$, or the contemporaneous effect of treatment~\cite{BG18},
can be defined and estimated analogous to \cref{eq:tensor-atet}. We focus on estimating an average quatity e.g. ATET instead of the mean squared
error of estimating the underlying tensor as this metric is oblivious of the choice of the underlying model. 
However, we do perform sensitivity analysis with respect to the parameters like rank of the tensor ($r$) and dependence on the history ($k$).

\subsection{Marginal Structural Models}

Our work builds on the {\em marginal structural models}, proposed by~\citet{RHB00}.  At each time $t$, for every possible sequence of treatments $a_{i,1:t}$, MSMs 
define the following model
\begin{equation}\label{eq:link-func}
\E{Y_{it}(a_{i,1:t})} = g(a_{i,1:t}, \beta)
\end{equation}
Here $g$ is the link function, usually chosen to be either a linear function or a logistic function.
Since there are time-varying confounders, the standard maximum likelihood based estimator of $\beta$ will be biased. \citet{Robins00} showed that the parameter can be
estimated  in an unbiased way through an inverse probability of treatment weighting (IPTW) approach. 
Suppose the observed data is given as $\{A_{i,t}, L_{i,t}, Y_{i,t}\}_{i,t}$. Then consider the following weight for each subject $i$ and  time period $t$:
$$\sw_{it} = \prod_{s=1}^t \frac{\Pro{A_{i,s}|A_{i,1:s-1}}}{\Pro{A_{i,s}|A_{i,1:s-1}, L_{i,1:s}} }$$
The denominator of each term is the probability of the corresponding treatment given the history up to that point. The numerator of each term is the marginal
probability of the corresponding treatment conditioned only on the past sequence of treatments and is used to stabilize the weights. Now if we compute a 
maximum likelihood estimator where the observation of subject $i$ at time $t$ is weighted by $\sw_{it}$, then  $\beta$ can be identified. 

\section{Estimation}
The goal is to design an {\em unbiased} and {\em consistent} estimator $\hatT$ of the $N \times T \times 2^k$ tensor $\T^\star_{N,T}$. 
We will assume that the tensor $\T^\star_{N,T}$ has low rank. Tensor $\T^\star_{N,T}$ has rank $r$ if there exist vectors $\{u_\ell \}_{\ell=1}^r$, $\{v_\ell\}_{\ell=1}^r$ and $\{w_\ell\}_{\ell=1}^r$ ($u_\ell \in \R^N, v_\ell \in \R^T, w_\ell \in \R^{2^k}$) such that
$
\T^\star_{N,T} = \sum_{\ell=1}^r u_\ell \otimes v_\ell \otimes w_\ell
$
and $r$ is the smallest integer such that $\T^\star_{N,T}$ can be written in this form. Here $u_\ell \otimes v_\ell \otimes w_\ell$ denotes the outer-product of the three vectors $u_\ell, v_\ell$, and $w_\ell$ with entries
$u_\ell \otimes v_\ell \otimes w_\ell (a,b,c) = u_\ell(a) \times v_\ell(b) \times w_\ell(c)$. Without loss of generality, we can assume that the tensor $\T^\star_{N,T}$ is written in the following form, where each of the vectors
$u_\ell, v_\ell$, and $w_\ell$ are normalized.
\begin{equation}\label{eq:norm-rep-tensor}
\T = \sum_{\ell=1}^r \lambda_\ell u_\ell \otimes v_\ell \otimes w_\ell
\end{equation}
We use $\lambda_\ell(\T^\star_{N,T}) = \lambda_\ell$ to denote the $\ell$-th singular value of $\T^\star_{N,T}$.
For $p = 1,\ldots,2^k$, let $S_p$ be the set of observations  leading to the realization of history corresponding to the $p$-th slice i.e.
$S_p = \set{(i,t) : a_{i,t-k+1:t} = p}$.
Then we wish to solve the following optimization problem:
\begin{equation}\label{eq:wt-least-square}
\min_{\substack{\T \in \R^{N\times T \times 2^k},\\ \rank(\T) \leq r} } \frac{1}{NT} \sum_{p = 1}^{2^k} \sum_{(i,t) \in S_p} w_{i,t} \left(Y_{i,t} - \T(i,t,p)\right)^2
\end{equation}
The weights $w_{i,t}$ are defined as:
\begin{equation}
w_{i,t} = \prod_{s=t-k+1}^t \frac{\Pro{A_{i,s} | A_{i,t-k+1:s-1}}}{\Pro{A_{i,s} | A_{i,t-k+1:s-1}, L_{i,1:t}} }.
\end{equation}
For each term, the numerator denotes the marginal probability of the treatment given the history of treatments from time $t-k+1$ to that time.
The denominator denotes the probability of the treatment given the history from time $t-k+1$ to that time and the additional covariates
$L_{i,1:t-k}$. 
%
But, why are we interested in the optimization problem \cref{eq:wt-least-square}? The objective function 
is the weighted log-likelihood given tensor $\T$, and we prove next that if we could solve this problem exactly, the corresponding estimator will be consistent. We make some additional assumptions. 
\begin{enumerate}[label=\textbf{A.\arabic*}]
        \item \label{asn:bdd}{\em Bounded Singular Value} : For each $N$ and $T$, each of the $r$ singular values of $\T^\star_{N,T}$ are bounded, i.e. $\norm{\T^\star_{N,T}}_{\star} = \max_{i} \abs{\lambda_i(\T^\star_{N,T})} \leq L$
        for some $L$.
        \item \label{asn:decay}{\em Decaying Covariance} : There exists a constant $\gamma < 1$, such that for all $t' > t+ k$, and for all sequences of treatments $a_{i,t-k+1:t}$ and $\tilde{a}_{i,t'-k+1:t'}$ and covariates $\ell_{i,1:t}$, we have
        \begin{equation}\label{eq:decaying-covariance}
        1-\varepsilon \le \frac{\Pro{\tilde{a}_{i,t'-k+1:t'}   | 
        \ell_{i,1:t} , a_{i,t-k+1:t}} }{ \Pro{\tilde{a}_{i,t'-k+1:t'}} } \le 1 + \varepsilon
        \end{equation}
        for $\varepsilon = \bigo{(t'-t)^{1-\gamma}}$.
\end{enumerate}

The first assumption implies that each entry of the tensor is bounded between $-L$ and $L$. The second assumption implies that the treatments chosen at two time periods
that are far apart, are almost independent. This  imposes a restriction on the  policy that generates the treatment sequences and does not impose any restriction on the evolution of the covariates.

\subsection{Consistency}

For a given $N$ and $T$, we will write $\widehat{\T}_{N,T}$ to denote the solution to \cref{eq:wt-least-square}.
Consider the weighted log-likelihood function:
\begin{equation}\label{eq:log-likelihood}
L_{N,T}(\T_{N,T}) = \frac{1}{NT} \sum_{p = 1}^{2^k} \sum_{(i,t) \in S_p} w_{i,t} \left(Y_{i,t} - \T(i,t,p)\right)^2
\end{equation}

The estimate $\widehat{\T}_{N,T}$ minimizes $L_{N,T}(\T_{N,T})$ over all possible choices of $\T_{N,T}$. Our goal is to show that with high probability, $\norm{\widehat{\T}_{N,T} - \T^\star_{N,T}}_2/\sqrt{NT}$
converges to zero as $N$ increases. We normalize the difference in norm by both $N$ and $T$.
This is necessary, as with increasing $N$, the number of parameters we are estimating also grows.
\begin{theorem}\label{thm:exact-mle}
	Suppose $\T^\star_{N,T}$ exists for all $N$ and $T$ and fix any $\eps > 0$.
	\begin{itemize}
		\item Suppose $k = \bigo{\log_{1/\delta} N}$. Then we have $\Pro{ \norm{\hatT_{N,T} - \T^\star_{N,T}}_2 / \sqrt{NT} > \eps } \rightarrow 0$ as $N \rightarrow \infty$.
		\item Suppose \ref{asn:decay} holds, and $k =  \bigo{\log_{1/\delta} T} $, then $\Pro{ \norm{\hatT_{N,T} - \T^*_{N,T}}_2 / {\sqrt{NT}} > \eps } \rightarrow 0$ as $T \rightarrow \infty$.
	\end{itemize}
\end{theorem}

The full proof is given in section~\ref{sec:proof_exact_mle} in the appendix. Here we sketch the main challenges.
The proof follows the ideas presented in~\citet{NM94}, but there are some subtle differences.  
First the
parameter space $\Theta_{N,T} = \set{\T \in \R^{N \times T \times 2^k} : \rank(\T) \leq r}$ need not be a closed set, as we can have a sequence
of rank $r$ tensors converging to a rank $r+1$ tensor \cite{Bini86}. However, the covexity of the log-likelihood function in $\T$ helps us to
circumvent this problem. Second, the standard way to prove the consistency of the maximum likelihood estimation is to consider a neighborhood
around the true parameter, say $\B$. Then there will be a gap of $\eps$ between the maximum over $\B$ and the maximum outside of $\B$,
and for large number of samples the gap between the objective value of the true parameter and the estimate will be less than $\eps$, and the
estimate will be inside the neighborhood $\B$. However, in our case, the gap $\eps$ is also changing with $N$ as the entire parameter
space is changing, and it might be possible that this gap goes to zero with increasing $N$. However, we can provide a lower bound 
on the gap in terms of the radius of the neighborhood and other parameters of the problem, and this helps to complete the proof.

\subsection{Solving Tensor Completion}
In this section, we focus on solving the weighted tensor completion to estimate the underlying tensor $\T^*_{N,T}$. First, we convert the weighted tensor completion problem to
a weighted tensor approximation problem with an additive error that goes zero as the number of units $N$  increases to infinity. 
This has two benefits. We can provide  a $(1+\eps)$-approximation to the weighted tensor approximation problem under reasonable assumptions on the  policy generating the treatment assignment. However, this algorithm is quite hard to implement to practice. So, we propose a gradient descent based algorithm for
the weighted tensor approximation problem. Compared to the original tensor completion problem, the gradients of the parameters are non-negative for
the unobserved entries of the tensor, and help the algorithm to converge faster. We need two definitions. 
Let us define the following tensor:
\begin{align*}
Y_{w}(i,t,p) = \left\{ \begin{array}{cc}
\frac{w_{i,t}Y_{i,t} }{\Pro{(i,t) \in S_p}} & \text{ if } (i,t) \in S_p  \\
0 & \text{ otherwise } 
\end{array}
\right.
\end{align*}
and the ``weight'' tensor, $W(i,t,p) = \sqrt{\Pro{(i,t) \in S_p} }$. This leads  to a   tensor approximation problem:
%
\begin{align}
\min_{\substack{\T \in \R^{N \times T \times 2^k}, \\ \rank(\T) \leq r, \norm{\T}_{\star} \leq L} } \frac{1}{NT} \norm{Y_w - \T}^2_W. \label{eq:new-objective}
\end{align}
Here $\norm{\T}^2_W$ denote the weighted Euclidean norm, i.e. $\norm{\T}^2_W = \sum_{i,j,k} W^2(i,j,k)\T^2(i,j,k)$.
Objective~\eqref{eq:new-objective} computes a weighted low rank approximation of $Y_{w}$. Let $\checkT_{N,T}$ be the solution
to~\eqref{eq:new-objective}. The next theorem show that $\checkT_{N,T}$ is also a consistent estimator.

\begin{theorem}\label{thm:approx-mle}
	Suppose $T^*_{N,T}$ exists for all $N$ and $T$. 
	\begin{itemize}
	\item If $k \leq \bigo{\log_{1/\delta} N}$, then for any $\eps > 0$, $\Pro{\norm{\checkT_{N,T} - \T^*_{N,T}}_2 / \sqrt{NT}  > \eps }  \rightarrow 0$ as $N \rightarrow \infty$.
	\item If $k \leq \bigo{\log_{1/\delta} T}$ and \ref{asn:decay} holds, then  $\forall \ \eps > 0$, $\Pro{\norm{\checkT_{N,T} - \T^*_{N,T}}_2 / \sqrt{NT}  > \eps }  \rightarrow 0$ as $T \rightarrow \infty$.
	\end{itemize}
\end{theorem}
The proof works by first showing that converting weighted tensor completion (\ref{eq:wt-least-square}) to weighted tensor approximation (\ref{eq:new-objective}) introduces
an error which goes to zero as $N$ increases to infinity. Therefore, $\checkT_{N,T}$ is an approximate minimizer of original problem~\ref{eq:wt-least-square}.
Then we can modify the proof of theorem \ref{thm:exact-mle} to show that such an approximate minimizer is also consistent. The full proof is given in section \ref{sec:proof_approx_mle} in the appendix.


\begin{figure*}[t!]
\centering
\begin{subfigure}[b]{0.48\textwidth}
\centering
\begin{tikzpicture}[scale=0.85]
	\node[state] at (0,0) (l1) {$L_{i,1}$};
	\node[state] at (2,0) (l2) {$L_{i,2}$};
	\node[state] at (4,0) (l3) {$L_{i,3}$};
	\node[state] at (1,1.5) (a1) {$A_{i,1}$};
	\node[state] at (3,1.5) (a2) {$A_{i,2}$};
	\node[state] at (5,1.5) (a3) {$A_{i,3}$};
	\draw[->] (l1) edge (a1);
	\draw[->] (l2) edge (a2);
	\draw[->] (l3) edge (a3);
	\draw[->] (a1) edge (l2);
	\draw[->] (a2) edge (l3);
	\draw[->] (a1) edge (a2);
	\draw[->] (a2) edge (a3);
    \draw[decorate sep={1mm}{3mm},fill] (6,1) -- (6.7,1);
\end{tikzpicture}
\caption{Simple Policy: Treatment $A_{i,t}$ depends on the current covariate $L_{i,t}$ and the immediate past treatment $A_{i,t-1}$.}
\end{subfigure}
~
\begin{subfigure}[b]{0.48\textwidth}
\centering
\begin{tikzpicture}[scale=0.85]
\node[state] at (0,0) (l1) {$L_{i,1}$};
	\node[state] at (2,0) (l2) {$L_{i,2}$};
	\node[state] at (4,0) (l3) {$L_{i,3}$};
	\node[state] at (1,1.5) (a1) {$A_{i,1}$};
	\node[state] at (3,1.5) (a2) {$A_{i,2}$};
	\node[state] at (5,1.5) (a3) {$A_{i,3}$};
	\draw[->] (l1) edge (a1);
	\draw[->] (l2) edge (a2);
	\draw[->] (l3) edge (a3);
	\draw[->] (a1) edge (l2);
	\draw[->] (a2) edge (l3);
	\draw[->] (a1) edge (a2);
	\draw[->] (a2) edge (a3);
	\draw[->] (l1) edge (a2);
	\draw[->] (l1) edge[bend left=85] (a3);
	\draw[->] (l2) edge (a3);
	\draw[->] (a1) edge[bend left] (a3);
	\draw[decorate sep={1mm}{3mm},fill] (6,1) -- (6.7,1);
\end{tikzpicture}
\caption{Complex Policy: Treatment $A_{i,t}$ depends on the covariates $\{L_{i,t'}\}_{t'=t-2}^{t}$ and treatments $\{A_{i,t'}\}_{t'=t-3}^{t-1}$.}
\end{subfigure}
\caption{\label{fig:policies}Different policies considered for  Simulation}
\end{figure*}
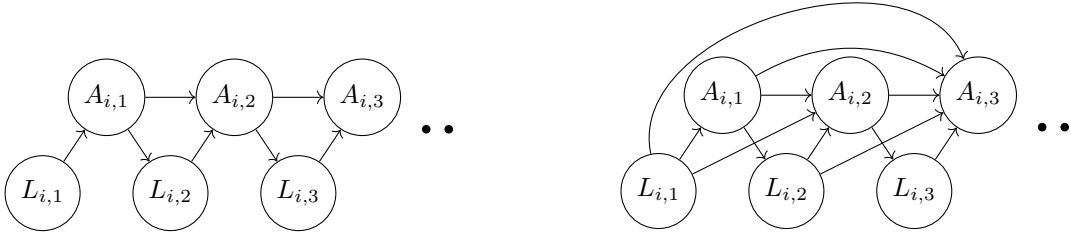

\subsubsection{A $(1+\varepsilon)$-approximation algorithm}

\citet{SWZ19} show that there is an algorithm that takes as input a tensor $A \in \R^{n \times n \times n}$, a weight tensor $W \in \R^{n \times n \times n}$, and outputs a tensor $A'$ of rank $r$ such that
$\norm{A - A'}^2_W \leq (1 + \eps) \min_{\rank(B) \leq r} \norm{A - B}^2_W.$ 
The authors consider the case when the weight tensor $W$ has $s$ distinct faces in  two dimensions (e.g. $s$ distinct rows, and columns). Then their algorithm runs in time $\nnz(A) + \nnz(W) + n2^{\tilde{O}\left(s^2r^2/\eps \right) }$ time, where $\nnz(A)$ is the
number of nonzero entries in $A$. 
%
%
%
We want to find a rank $r$ approximation of tensor $Y_{w} \in \R^{N \times T \times 2^k}$.
The main challenge in applying the algorithm proposed by \citet{SWZ19} is that we want to ensure that 
the singular values are bounded between $-L$ and $L$. This can be handled by introducing $r$ additional constraints in the polynomial
system verifier of the algorithm in \cite{SWZ19}. We provide the full algorithm and an anlysis of its running time  in section~\ref{sec:mult_approx} in the appendix.

\subsubsection{Projected Gradient Descent}
We now provide a simple algorithm for the weighted tensor approximation problem~\eqref{eq:new-objective} based on  projected gradient descent.
Algorithm ~\ref{alg:proj-grad-descent} 
repeatedly applies two steps. Line 5 computes a gradient step to compute the new tensor $\T_u$. 
However, the tensor $\T_u$ might not be of rank $r$, so line 6 computes a projection of tensor $\T_u$ into the space of tensors of rank $r$.
As the projection step is a standard rank $r$ approximation of a tensor, we use the {\em parafac} method from the TensorLy package \cite{KPAP18}
for this step.
\begin{algorithm}
	\caption{Weighted Tensor Approximation}
	\label{alg:proj-grad-descent}
	\begin{algorithmic}[1]
		\STATE {\bfseries Input:} Tensor $\bS \in \R^{N \times T \times 2^k}$, weight tensor $W \in \R^{N \times T \times 2^k}$, rank $r$, and $R$.
		\STATE Initialize $\T$.
		\FOR{$j=1$ to $R$}
		\STATE $\T_u \leftarrow \T + \lambda 2 W^2 (\bS - \T)$
		\STATE $\T \leftarrow \text{Project}(\T_u,r)$
		\IF{Relative Change in Loss $\le \eps$} 
		\RETURN $\T$
		\ENDIF
		\ENDFOR
		\RETURN $\T$
	\end{algorithmic}
\end{algorithm}

\section{Simulation}
\begin{figure*}[t!]
\centering
\begin{subfigure}[b]{0.48\textwidth}
\centering
\includegraphics[scale=0.37]{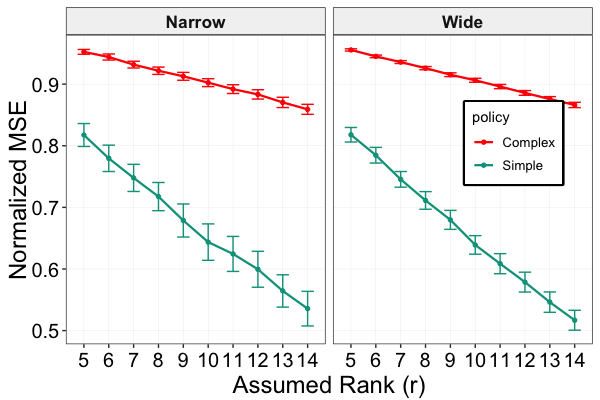}
\end{subfigure}
~
\begin{subfigure}[b]{0.48\textwidth}
\centering
\includegraphics[scale=0.37]{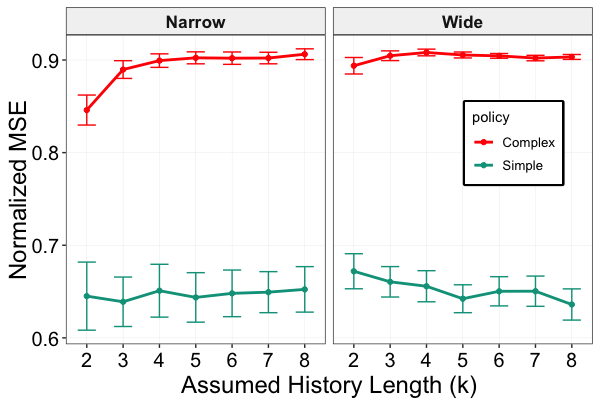}
\end{subfigure}
\caption{Sensitivity of algorithm \ref{alg:proj-grad-descent} to assumed rank ($r$) of the underlying tensor, and the history length ($k$). The true outcome model is 
generated with tensor of rank $r=10$, and temporal dimension $k=5$. The treatment is generated according to two different policies (simple and complex). Error bars show standard
errors from repeating each simulation 100 times. (a) As the parameter $r$ increases,  normalized MSE drops significantly, implying a better fit with tensors of rank $> 10$. (b) However, the error seems to be robust to changes in  $k$. \label{fig:combined}}
\end{figure*}

We now evaluate the effectiveness of Algorithm \ref{alg:proj-grad-descent} through a simulation.\footnote{The code for the simulation is available at \url{https://github.com/debmandal/Tensorized_MSM}} We consider the simulation setup introuduced by \cite{IR15} and consider two types of worlds --  narrow and wide. The narrow world has more agents compared to the number of time steps ($N=500$ and $T=10$), whereas the wide world has fewer agents compared to the number of time steps ($N=10$ and $T=500$). We consider these two worlds to see how our algorithm performs when either the number of agents $N$ or the number of time-steps $T$ is large compared to the other parameter.

We generated the data i.e. the treatment assignment $\{A_{i,t}\}_{i,t}$ and the outcome $\{Y_{i,t}\}_{i,t}$ according to  two policies. Both of them are adapted from \cite{IR15}, however \citet{IR15} considered only three time-steps, whereas we generalize the treatment policy for an arbitrary number time-steps. We provide full details of the policies for completeness, and also to highlight the differences with~\cite{IR15}. Figure~\ref{fig:policies} shows the causal models underlying the two policies. 
\begin{enumerate}
	\item {\bf Simple}: The treatment at time period $t$, $A_{i,t}$ depends on the current covariate $L_{i,t}$ and the immediate past treatment $A_{i,t-1}$. Specifically, we write the covariates as $L_{i,t} = (Z_{it1}\cdot U_{it}, Z_{it2} \cdot U_{it}, \abs{Z_{it3} \cdot U_{it}}, \abs{Z_{it4} \cdot U_{it}})^\top$. Here $Z_{itk}$ is an iid draw from the standard normal distribution, $U_{i1} = 1$, and $U_{it} = 2 + (2A_{i,t-1} - 1) / 3$ for $t \geq 2$. The treatments are generated as $P(A_{i,t} = 1) = \textrm{expit}\left\{-A_{i,t-1} + \gamma^\top L_{i,t} + (-1/2)^t \right\}$, where $\gamma = (1, -0.5, 0.25, 0.1)^\top$. Additionally, we set $A_{i,0} = 0$ for generating the treatment at time $t=1$.
	\item {\bf Complex}: The treatment at time period $t$, $A_{i,t}$ depends on the covariate sequence $\{L_{i,t'}\}_{t'=t-2}^t$ of length three, and past treatment sequence $\{A_{i,t'}\}_{t'=t-3}^{t-1}$ of length three. Like the simple policy, we write covariate $L_{i,t}$ as $L_{i,t} = (Z_{it1}\cdot U_{it}, Z_{it2} \cdot U_{it}, \abs{Z_{it3} \cdot U_{it}}, \abs{Z_{it4} \cdot U_{it}})^\top$, where $Z_{itk}$ is an iid draw from the standard normal distribution. However the definition of $U_{it}$-s are changed as $U_{i1} = 1$ and $U_{it} = \prod_{t'=t-3}^{t-1}\left\{2 + (2A_{i,t'} - 3)\right\}$ for $t \ge 2$. The treatments are generated as $P(A_{i,t} = 1) = \textrm{expit}\left\{\sum_{t'=t-2}^{t} A_{i,t'-1} + \gamma^\top L_{i,t'} + (-1/2)^t \right\}$, where $\gamma = (1, -0.5, 0.25, 0.1)^\top$. Additionally, we set $A_{i,0} = A_{i,-1} = A_{i,-2} = 0$ for generating the treatment at time $t=1$.
\end{enumerate}

\noindent \textbf{Outcome Model}: In order to generate the outcome variables $\{Y_{i,t}\}_{i,t}$ we first fix a tensor $\T \in \R^{N \times T \times 2^k}$. The tensor is used to introduce the desired heterogeneity in the potential outcomes, and is chosen as follows. 
Fix rank $r=10$, and choose the vectors  $\{u_i\}_{i=1}^r$, $\{v_i\}_{i=1}^r$ and $\{w_i\}_{i=1}^r$ 
by selecting each entry uniformly at random
from the interval $[0,1]$ and then normalizing the vectors. Second, we select the singular values $\{\lambda_i\}_{i=1}^r$ uniformly at random from the interval $[50,200]$. This gives us a tensor $\T = \sum_{i=1}^r \lambda_i u_i \otimes v_i \otimes w_i$. Having fixed the tensor, we generate the outcome $Y_{i,t}$ at time $t$ as
\begin{align}
Y_{i,t} =  250 - 10 \sum_{t'=t-2}^t A_{i,t'} + \sum_{t'=t-2}^t \delta^\top L_{i,t'}  + \T(i,t, A_{i,t-k+1:t}) + \nu_{i,t} \label{eq:outcome-model-simulation}
\end{align}
where $\delta = (1, -0.5, 0.25, 0.1)^\top$ and $\nu_{i,t}$ is an iid draw from the standard normal distribution. We introduce the extra term $\T(i,t, A_{i,t-k+1:t})$ to the original outcome model considered in \citet{IR15}. Also recall that the largest singular value of $\T$ is at most $200$, so that the new term does not dominate the rest of the outcome model.

Since algorithm ~\ref{alg:proj-grad-descent} needs to know the parameters $r$ and $k$, we first observe how sensitive it  is to the choice of the assumed rank parameter $r$ and the assumed length of the history $k$. Figure ~\ref{fig:combined} plots the normalized  mean squared error (MSE) for various choices of $r$ and $k$ values for the two types of policies. As rank $r$ increases, the error goes down significantly. This implies that even though a tensor of rank $10$ is used in the outcome model, a tensor of higher rank might be a better fit for the marginal outcomes. On the other hand, the error seems to be quite robust to changes in the parameter $k$. However, we believe that one should see a drop in the error with higher $k$ if the outcome model is more heterogeneous e.g. the tensor $\T$ is more dominant in the outcome model ~\eqref{eq:outcome-model-simulation}.

Finally, we fit traditional MSM \cite{Robins00} at every time-step $t$. Since, conditioned on the covariates and treatments, the outcome is distributed according to a normal distribution
, we use a linear function as the link function, i.e., $g(a_{i,1:t},\beta^t) = \left \langle a_{i,1:t}, \beta^t\right \rangle$ in eq. \ref{eq:link-func}. Then we solve a weighted least squares regression problem to obtain the parameters $\{\beta^t\}_{t \in [T]}$. For the narrow world ($N=500$ and $T=10$), MSM performs reasonably well, and the normalized MSE turns out to be {\bf 8.69} (resp. {\bf 8.31}) for the simple (resp. complex) policy.  Algorithm \ref{alg:proj-grad-descent}, on the other hand,
gives much better performance and has normalized MSE of {\bf 0.64} (resp. {\bf 0.90}) for the simple (resp. complex) policy with parameters $r=10$ and $k=5$. We also evaluated classical MSM on the wide world ($N=10$ and $T=500$), but it performs poorly and has normalized MSE greater than $10^6$. This highlights a main drawback of the MSM -- traditional methods don't perform well if the number of time-steps $T$ is large compared to $N$.

\section{Conclusion and Future Work}
In this work, we proposed a new type of marginal structural models based on tensors, and showed how to estimate the parameters of the model. There are many interesting directions for future work. We assumed perfect knowledge of the policies in order to estimate the weights. So, it would be interesting to learn the weights from data and develop a 
doubly robust estimator~\cite{BR05}, which works if either the outcome or the treatment model is mis-specified. Furthermore, an interesting direction is to consider the presence of unobserved confounders along the lines of \citet{BAS19}, who developed a deconfounder for time-varying treatments. Finally, it will be interesting to see if we can generalize the results of \cite{CRY19} and theoretically analyze the performance of algorithm~\ref{alg:proj-grad-descent}. 

\bibliography{references}
\bibliographystyle{plainnat}

\appendix
\section{Proof of Theorem \ref{thm:exact-mle}\label{sec:proof_exact_mle}}
\begin{proof}
Let $S(i,t)$ be the random variable which denotes the length $k$-history of user $i$ at time $t$. Then, the weighted log-likelihood function with respect to a tensor $\T_{N,T}$ is given as :
	$$L_{N,T}(\T_{N,T}) = \frac{1}{NT} \sum_{i=1}^N \sum_{t=1}^T w_{i,t} (Y_{i,t} - \T_{N,T}(i,t,S(i,t)))^2.$$
	First we compute the expected value of the weighted log-likelihood with respect to the policy $\mathcal{P}$ (i.e. the random
	variables $\{Y_{i,1:T}\}_{i=1}^N, \{A_{i,1:T}\}_{i=1}^N, \{L_{i,1:T}\}_{i=1}^N$ and the true underlying tensor $\T^\star_{N,T}$.
	We write $\ell^*_{N,T}(\T_{N,T})$ to denote this quantity as it only depends on the tensor $\T_{N,T}$, i.e. $\ell^*_{N,T}(\T_{N,T}) = \Em{\mathcal{P},\T^\star_{N,T}}{L_{N,T}(\T_{N,T})}$. 

 We want to show that $\norm{\hatT_{N,T} - \T^\star_{N,T}}$ becomes small as either $N$ or $T$ increases. Our proof is based on the proof of the consistency
	of the maximum likelihood given in \cite{NM94}. 
	We write $\Theta_{N,T}$ to denote the parameter space $\set{\T \in \R^{T \times N \times 2^k} : \rank(\T) \leq r, \norm{\T}_{\infty} \leq L}$. $\Theta_{N,T}$ is bounded but need not be closed because
	of issues with
	border tensor. It is known that there might exist a sequence of rank $r$ tensors whose limit is a rank $r+1$
	tensor \cite{Bini86}. However, we can exploit the convexity of $L_{N,T}(\cdot)$ to overcome this problem.

	First consider a neighborhood $\B$ of radius $d$ centered at $\T^\star_{N,T}$ and contained within the interior
	of $\Theta_{N,T}$. 
	$$\B = \set{\T \in \R^{N \times T \times B} : \norm{\T - \T^\star_{N,T}}_2/\sqrt{NT} \leq d}$$
	Lemma \ref{lem:concave-likelihood} proves that $L_{N,T}(\cdot)$ is convex
	over $\Theta_{N,T}$. Since a convex function is continuous over the interior of its domain, $L_{N,T}(\cdot)$ is continuous over $\B$.  Moreove, unlike $\Theta_{N,T}$, set $\B$ is a compact set. This implies that there exists 
	a minimizer for $L_{N,T}(\cdot)$ over $\B$. Suppose $\tildeT_{N,T}$ be the minimizer of 
	$L_{N,T}(\cdot)$ over $\B$. Consider any $\T \in \Theta_{N,T} \setminus \B$. Then there exists $\lambda < 1$ such that
	$\T' = \lambda \tildeT_{N,T}+ (1-\lambda) \T$ and $\T' \in \B$. This gives us the following :
	\begin{align*}
	&L_{N,T}(\tildeT_{N,T}) \leq L_{N,T}(\T') = L_{N,T}(\lambda \tildeT_{N,T}+ (1-\lambda) \T) \leq \lambda L_{N,T}(\tildeT_{N,T}) + (1-\lambda) L_{N,T}(\T) \\
	&\Rightarrow L_{N,T}(\tildeT_{N,T}) \leq  L_{N,T}(\T)
	\end{align*}
	This first line uses the convexity of $L_{N,T}(\cdot)$ (lemma \ref{lem:concave-likelihood}).
	This proves that $\tildeT_{N,T}$ is actually the minimizer of $L_{N,T}(\cdot)$ over the entire parameter space $\Theta_{N,T}$.
	Moreover, any other minimizer of $L_{N,T}(\cdot)$ must be inside $\B$. Otherwise, suppose $\T''$ minimizes $L_{N,T}(\cdot)$ and $\T'' \in \Theta_{N,T}\setminus \B$.
	Then for any $0<p<1$ and $\eps = \delta^{2k} d^2$ we have with probability at least $1 - \bigo{1/\eps^2 N^p}$,
	\begin{align*}
	L_{N,T}(\tildeT_{N,T}) - \eps / 3  < L_{N,T}(\T^\star_{N,T}) - \eps / 3 <  \ell^*_{N,T}(\T^\star_{N,T}) < \ell^*_{N,T}(\T'') - \eps  <  L_{N,T}(\T'') - 2\eps/3
	\end{align*}
	The second and the fourth inequality uses lemma \ref{lem:variance-bound} and the third inequality uses lemma
	\ref{lem:gap-likelihood}. Therefore, with probability at least $1 - \bigo{1/(d^4 \delta^{4k} N^p)}$ all the maximizers of $L_{N,T}(\cdot)$ must be inside the ball $\B$ as long as $k = \bigo{\log_{1/\delta} N}$. This proves that for any $d$ we 
	can choose $N$ large enough (possibly dependent on $\delta$) such that with  probability at least $1 - 1/\poly(N)$ the minimizer of $L_{N,T}(\cdot)$ lies within a $d$ neighborhood of $\T^\star_{N,T}$. This proves the consistency of
	the estimate when $N$ increases to inifinity. The proof of consistency when the number of time periords $T$ increases to infinity is similar.
	
 \end{proof}

\section{Proof of theorem \ref{thm:approx-mle}\label{sec:proof_approx_mle}}

\begin{proof}
	Lemma \ref{lem:approx-optim} shows that $\checkT_{N,T}$ apporximately optimizes the original objective~\ref{eq:wt-least-square} i.e.
	$$\frac{1}{NT}\sum_{p = 1}^{2^k} \sum_{(i,t) \in S_p} w_{i,t} \left(Y_{i,t} - \checkT_{N,T}(i,t,p)\right)^2 \leq  \opt + o_p(1)$$
	where the error term $o_p(1)$ term goes to zero as $N$ goes to infinity. Now we proceed similar to the proof of theorem \ref{thm:exact-mle}.
	
	We write $\Theta_{N,T}$ to denote the parameter space i.e. $\set{\T \in \R^{T \times N \times 2^k} : \rank(\T) \leq r, \norm{\T}_{\infty} \leq L}$. 
	First consider a neighborhood $\B$ of radius $d$ centered at $\T^\star_{N,T}$ and contained within the interior
	of $\Theta_{N,T}$. 
	$$\B = \set{\T \in \R^{N \times T \times B} : \norm{\T - \T^\star_{N,T}}_2/\sqrt{NT} \leq d}$$
	Suppose $\tildeT_{N,T}$ is an approximate  minimizer of 
	$L_{N,T}(\cdot)$ over $\B$, i.e.
	$$
	L_{N,T}(\tildeT_{N,T}) \leq \min_{\T\in \B} L_{N,T}(\T) + o_p(1)
	$$
	 Consider any $\T \in \Theta_{N,T} \setminus \B$. Then there exists $\lambda < 1$ such that
	$\T' = \lambda \tildeT_{N,T}+ (1-\lambda) \T$ and $\T' \in \B$. This gives us the following :
	\begin{align*}
	&L_{N,T}(\tildeT_{N,T}) \leq L_{N,T}(\T') + o_p(1) = L_{N,T}(\lambda \tildeT_{N,T}+ (1-\lambda) \T) + o_p(1) \\&\leq \lambda L_{N,T}(\tildeT_{N,T}) + (1-\lambda) L_{N,T}(\T) + o_p(1)\\
	&\Rightarrow L_{N,T}(\tildeT_{N,T}) \leq  L_{N,T}(\T) + o_p(1)
	\end{align*}
	This first line uses the convexity of $L_{N,T}(\cdot)$ (lemma \ref{lem:concave-likelihood}).
	This proves that $\tildeT_{N,T}$ is actually an approximate minimizer of $L_{N,T}(\cdot)$ over the entire parameter space $\Theta_{N,T}$.
	Moreover, any other any other approximate minimizer of $L_{N,T}(\cdot)$ must be inside $\B$. 
	Otherwise, suppose $\T''$ approximately minimizes $L_{N,T}(\cdot)$ and $\T'' \in \Theta_{N,T}\setminus \B$.
	Then for $\eps = \delta^{2k} d^2$ we have with probability at least $1 - \bigo{1/\eps^2 N^p}$,
	\begin{align*}
	&L_{N,T}(\tildeT_{N,T}) - \eps / 3  < L_{N,T}(\T^\star_{N,T}) + o_p(1) - \eps / 3 < \ell^*_{N,T}(\tildeT_{N,T}) + o_p(1) \\ &< \ell^*_{N,T}(\T'') - \eps + o_p(1) <  L_{N,T}(\T'') - 2\eps/3 + o_p(1)
	\end{align*}
	The second and the fourth inequality uses lemma \ref{lem:variance-bound} and the third inequality uses lemma
	\ref{lem:gap-likelihood}. This gives us the following:
	\begin{align*}
		L_{N,T}(\tildeT_{N,T}) < L_{N,T}(\T'') - \eps / 3 + o_p(1)
	\end{align*}
	As $\eps$ is independent of $N$, this tells us that $\T''$ cannot be an approximate minimizer of $L_{N,T}(\cdot)$. 
	Therefore, with probability at least $1 - \bigo{1/(d^4 \delta^{4k} N^p)}$ all the approximate minimizers of $L_{N,T}(\cdot)$ must be inside the ball $\B$. This proves that as long as $k = \bigo{\log_{1/\delta}T}$, for any $d$ we 
	can choose $N$ large enough (possibly dependent on $\delta$) such that with  probability at least $1 - 1/\poly(N)$ any approximate  minimizer of $L_{N,T}(\cdot)$ lies within a $d$ neighborhood of $\T^\star_{N,T}$. The proof when $N$ if fixed and $T$ increases to infinity is similar if we use the second part of lemma \ref{lem:approx-optim}.
\end{proof}

\section{A $(1+\epsilon)$-multiplicative approximation \label{sec:mult_approx}}
In this section, we provide the details of the $(1+\eps)$-approximation algorithm for weighted tensor approximation. 
We will write $B$ to denote $2^k$. As input, we are given
a tensor $\T \in \R^{N \times T \times B}$, a weight tensor $W \in \R^{N \times T \times B}$ and our goal is to  solve 
\begin{align}
	\min_{\substack{B :\rank(B) \le r \\ \norm{B}_{\star} \leq L}} \norm{\T - B}^2_{W} \label{eq:new-objective-app}
\end{align}
Suppose we are guaranteed that $W$ has $s$ distinct rows and $s$ distinct columns. This also guarantees that the number of distinct tubes of $s$ is at most $S = 2^{\bigo{s \log s}}$.

\begin{algorithm}[h]
	\caption{Weighted Low Rank Tensor Approximation}
	\label{alg:example}
	\begin{algorithmic}[1]
		\STATE {\bfseries Input:} Tensor $\T \in \R^{N \times T \times B}$, weight tensor $W \in \R^{N \times T \times B}$, rank $r$, rank of weight tensor $s$, and $\epsilon$.
		\STATE {\bfseries Output:} Tensor $\T'$ of rank $k$ such that $\norm{\T - \T'}^2_{W} \leq (1 + \eps) \min_{B :\rank(B) \le r, \norm{\T}_{\star} \leq L} \norm{\T - B}^2_{W}.$
		\FOR{$j=1$ to $3$}
		\STATE $s_j \leftarrow \bigo{r/\epsilon}$
		\ENDFOR
		\STATE Choose three sketching matrices $S_1 \in \R^{TB\times s_1}, S_2 \in \R^{NB\times s_2}$, and $S_3 \in \R^{NT\times s_3}$
		\FOR{$j=1$ to $2$}
		\STATE \COMMENT{Omitting the third dimension}
		\FOR{$i=1$ to $s$}
		\STATE{Create $r\times s_j$ variables for matrix $P_{i,j} \in \R^{r \times s_j}$}
		\STATE{Set $(\hat{U}_j)^i = \T^j_i D_{W^j_i} S_j P^T_{j,i} (P_{j,i} P^T_{j,i})^{-1}$}
		\ENDFOR
		\ENDFOR
		\FOR{$i=1$ to $S$}
		\STATE \COMMENT{Representing the third dimension}
		\STATE{Set $(\hat{U}_3)^i = \T^3_i D_{W^3_i} S_j P^T_{3,i} (P_{3,i} P^T_{3,i})^{-1}$}
		\ENDFOR
		\STATE{Form $\norm{W \cdot (\hat{U}_1 \otimes \hat{U}_2 \otimes \hat{U}_3 - \T)}^2_F$}
		\FOR{$i=1$ to $r$}
		\STATE{Add constraint $\norm{\hat{U}_1^i}^2_2 \norm{\hat{U}_2^i}^2_2 \norm{\hat{U}_3^i}^2_2 \leq L$}
		\ENDFOR
		\STATE{Run Polynomial System Verifier to get $U_1, U_2,$ and $U_3$}
		\STATE \RETURN $U_1\otimes U_2 \otimes U_3$ 
	\end{algorithmic}
\end{algorithm}

Algorithm \ref{alg:example} closely follows algorithm G.4 in \cite{SWZ19} with modifications to handle asymmetric tensors and additional constraint on the bound for the largest singular value.
It chooses three sketching matrices of appropriate dimension to solve the original low-rank approximation problem in a low-dimensional space. The main idea is that the entries of $\hat{U}_1$ 
can be repersented with as polynomials of the variables for $i=1$ to $s$ (line 10). This is possible because the weight matrix has $s$ distinct rows and columns, which implies that it's flatenning along
the rows has $s$ distinct faces. The same thing holds for $\hat{U}_2$. However, this need not be true for $\hat{U}_3$, so they are represented through $S$ distinct denominators (line 16). With this
setup \cite{SWZ19} shows that the number of variables in the polynomial system verifier is $\bigo{r^2 s/ \eps}$ and the number of constraints is $2s + S$. In line 18, we add additional $r$ constraints.
So the total number of constraints is $2s + r + 2^{\bigo{s \log s}}$ and the total number of variables is  $\bigo{r^2 s/ \eps}$. Moreover, the degree of the new constraints in line 18 is at most $\poly(r,s,S)$.
A polynomial system can be verified in time $(\text{\# max degree of any polynomial})^{\text{\# number of variables}}$. In our case, this takes time
$$\left( \poly(r,s)\poly\left( 2^{\bigo{s \log s} }\right) \right)^{\bigo{r^2 s / \eps}} = \left( \poly(r,s)2^{\bigo{s \log s} } \right)^{\bigo{r^2 s / \eps}} = 2^{\tilde{O}\left(r^2 s^2 / \eps \right) }.$$

Recall that we want to compute a low-rank approximation
of the tensor $Y_w \in \R^{N \times T \times 2^k}$. Although $\nnz(Y_w) = NT$, 
positivity implies that the number of nonzero entries in $W$ is $\nnz(W) = NT2^k$. Therefore, the
resulting algorithm runs in time
time $\bigo{NT2^k + \max\{N,T,2^k\} 2^{\tilde{O}\left(s^2r^2/\eps\right)} }$ and outputs a 
tensor $\tildeT_{N,T}$ such that
$\norm{Y_w - \tildeT_{N,T} }^2_W \leq (1 + \eps) \min_{\T \in \R^{N \times T \times 2^k}, \rank(\T) \leq r} \norm{Y_w - \T}^2_W$ with probability at least $9/10$.


\subsection{Distinct Faces of the Weight Matrix}
Recall that we need the weight matrix $W$ to have $s$ distinct faces in two dimensions, where the
weight matrix $W$ is defined as $W(i,t,p) = \sqrt{\Pro{(i,t) \in O_p} }$. If the underlying policy satisfies the following two assumptions, then the matrix $W$ has $s$ distinct faces along the two dimensions.

\begin{enumerate}
	\item There are {\bfseries $s$ groups of subjects}
	such that the policy treats all the subjects in a group identically.
	\item There are {\bfseries $s$ groups of time periods} such that for any two
	time $t$ and $t'$ belonging to the same group we have
	the same marginal probabilities across all the
	subjects ($\Pro{(i,t) \in O_p} = \Pro{(i,t') \in O_p}\
	\forall i,p$).
\end{enumerate}
These two assumptions together imply that $W$ has $s$ distinct faces in two
dimensions, and allows an efficient $(1+\eps)$-multiplicative approximation of the tensor approximation problem defined in equation~\ref{eq:new-objective-app}.

\section{Additional Lemmata}
	\begin{lemma}\label{lem:variance-bound}
	Suppose $\T^\star_{N,T}$ exists for all $N$ and $T$. 
\begin{itemize}
	\item If $k \leq  \bigo{\log_{(1-\delta)/\delta} N}$, $\Pro{\abs{L_{N,T}(\T) - \ell^*_{N,T}(\T)} > \epsilon} \rightarrow 0$ as $N \rightarrow \infty$.
	\item If $k \leq  \bigo{\log_{(1-\delta)/\delta} N}$, and \ref{asn:decay} holds $\Pro{\abs{L_{N,T}(\T) - \ell^*_{N,T}(\T)} > \epsilon} \rightarrow 0$ as $T \rightarrow \infty$.
\end{itemize}
	
\end{lemma}
\begin{proof}
	We will write $S(i,t)$ to denote the history of length $k$ for user $i$ at time $t$. With this notation, our objective function becomes,
	\begin{align*}
		L_{N,T}(\T) = \frac{1}{NT} \sum_{i,t} w_{i,t} \left(Y_{i,t} - \T(i,t,S(i,t))\right)^2
	\end{align*}
	\begin{align*}
	&\Pro{\abs{L_{N,T}(\T) - \ell^*_{N,T}(\T)} > \eps} \leq \frac{\var{L_{N,T}(\T)}}{\eps^2} = \frac{1}{\eps^2 N^2 T^2} \var{\sum_{i,t} w_{i,t} \left(Y_{i,t} - \T(i,t,S(i,t))\right)^2} \\
	&= \frac{1}{\eps^2 N^2 T^2} \left[ \sum_{i,t} \var{w_{i,t} \left(Y_{i,t} - \T(i,t,S(i,t))\right)^2} \right. \\
	&+ \left.  2  \sum_{i, t < t'} \textrm{cov}\left( w_{i,t}\left(Y_{i,t} - \T(i,t,S(i,t))\right)^2, w_{i,t'} \left(Y_{i,t} - \T(i,t',S(i,t'))\right)^2\right)\right]\\
	&\leq \frac{1}{\eps^2 N^2 T^2} \left[ \sum_{i,t} \E{w^2_{i,t} \left(Y_{i,t} - \T(i,t,S(i,t))\right)^4}  \right. \\ &+ 2  \left. \sum_{i, t < t'\leq t + k} \E{ w_{i,t} w_{i,t'}\left(Y_{i,t} - \T(i,t,S(i,t))\right)^2\left(Y_{i,t'} - \T(i,t',S(i,t'))\right)^2}  \right. \\
	&+ \left. 2 \sum_{i, t + k < t'} \textrm{cov}\left( w_{i,t} \left(Y_{i,t} - \T(i,t,S(i,t))\right)^2, w_{i,t'} \left(Y_{i,t'} - \T(i,t',S(i,t'))\right)^2\right)  \right] \\
	\end{align*}
	Now we bound each term in the last summation. First, consider the case when $T$ is fixed and $N$ goes to infinity. Since weightes are
	bounded by $\left( \frac{1-\delta}{\delta} \right)^k$ and the fourth moments of the counterfactual outcomes are bounded, there exists a constant $M_1 > 0$
	such that $\E{w^2_{i,t} \left(Y_{i,t} - \T(i,t,S(i,t))\right)^4}  \leq \left( \frac{1-\delta}{\delta} \right)^{2k}M_1$. By a similar argument we can bound
	the second term 
	$\E{ w_{i,t} w_{i,t'}\left(Y_{i,t} - \T(i,t,S(i,t))\right)^2\left(Y_{i,t'} - \T(i,t',S(i,t'))\right)^2} $ by $ \left( \frac{1-\delta}{\delta} \right)^{2k}M_1$. Finally, we can
	bound the covariance term by the expectation of the products and get a similar bound. This gives us the following bound on the probability:
	
	\begin{align*}
	&\Pro{\abs{L_{N,T}(\T) - \ell^*_{N,T}(\T)} > \eps} \leq \frac{1}{\eps^2 N^2 T^2} \left[NT M_1\left( \frac{1-\delta}{\delta} \right)^{2k} + 2NT^2 M_1 \left( \frac{1-\delta}{\delta} \right)^{2k}\right] \\&= \bigo{\frac{1}{\eps^2 N^{p}}}
	\end{align*}
	for any $0 < p < 1$ as long as $k = \bigo{(1-p)\log_{(1-\delta)/\delta} N}$. This gives us the first result.

	Now consider the case when $T$ increases to infinity and $N$ is fixed. As before we bound $\E{w^2_{i,t} \left(Y_{i,t} - \T(i,t,S(i,t))\right)^4}$ by $\left( \frac{1-\delta}{\delta} \right)^{2k}M_1$. When $t < t' \le t+k$ we bound $\E{ w_{i,t} w_{i,t'}\left(Y_{i,t} - \T(i,t,S(i,t))\right)^2\left(Y_{i,t'} - \T(i,t',S(i,t'))\right)^2} $ by $ \left( \frac{1-\delta}{\delta} \right)^{2k}M_1$ and there are $Tk$ such terms. On the other hand, if $t' > t+k$ and \ref{asn:decay} holds lemma \ref{lem:covariance-bound} proves a bound of $O(\left(t'-t\right)^{1-\gamma} )$ on the covariance term. This gives us the following bound on the probability:

	\begin{align*}
	&\Pro{\abs{L_{N,T}(\T) - \ell^*_{N,T}(\T)} > \eps} \leq \frac{1}{\eps^2 N^2 T^2} \left[NT M_1\left( \frac{1-\delta}{\delta} \right)^{2k} + 2kT M_1 \left( \frac{1-\delta}{\delta} \right)^{2k} + \sum_{t: t' > t+k} c \left(t'-t\right)^{1-\gamma} \right] \\&= \bigo{\frac{1}{\eps^2 T^{p}}}
	\end{align*}
	for any $0 < p < \gamma$ as long as $k = \bigo{(1-p)\log_{(1-\delta)/\delta} N}$. This establishes the second result.
%
\end{proof}

\begin{lemma}\label{lem:wtlogit}
	\begin{align*}
	&\Em{H_{i,1:t}}{w_{i,t} (Y_{i,t} - \T(i,t,S(i,t))^2)} = \sum_{A_{i,t-k+1:t}} \Pro{A_{i,t-k+1:t}}^2 \sum_{L_{i,1:t}}  \Pro{L_{i,1:t}| A_{i,t-k+1:t}}\times \\
	&\sum_{Y_{i,t}}  \Pro{Y_{i,t}| A_{i,t-k+1:t}, L_{i,1:t}} 
	\left(Y_{i,t}(A_{i,t-k+1:t}, L_{i,1:t}) - \T(i,t,A_{i,t-k+1:t})\right)^2 \\
	\end{align*} 
\end{lemma}
\begin{proof}
	We assume that the outcome variable is discrete. The proof for continuous variable is similar.
	\begin{align*}
	&\Em{H_{i,1:t}}{w_{i,t} (Y_{i,t} - \T(i,t,S(i,t))^2)}\\ 
	&= \sum_{A_{i,1:t}} \sum_{L_{i,1:t}} \sum_{Y_{i,1:t}} \Pro{A_{i,1:t}, L_{i,1:t}, Y_{i,1:t}}  \prod_{s=t-k+1}^t \frac{\Pro{A_{i,s} | A_{i,t-k+1:s-1}}}{\Pro{A_{i,s} | A_{i,t-k+1:s-1}, L_{i,1:s} } } \times \\
	&\left(Y_{i,t}(A_{i,t-k+1:t}, L_{i,1:t}) - \T(i,t,A_{i,t-k+1:t})\right)^2 \\
	&= \sum_{A_{i,1:t}} \sum_{L_{i,1:t}} \sum_{Y_{i,1:t}}\Pro{A_{i,1:t}, L_{i,1:t}, Y_{i,1:t}}   \frac{\Pro{A_{i,t-k+1:t}}}{\Pro{A_{i,t-k+1:t} | L_{i,1:t}}} \times\\
	&\left(Y_{i,t}(A_{i,t-k+1:t}, L_{i,1:t}) - \T(i,t,A_{i,t-k+1:t})\right)^2 \\
	&= \sum_{A_{i,t-k+1:t}} \sum_{L_{i,1:t}} \sum_{Y_{i,t}} \Pro{A_{i,t-k+1:t}, L_{i,1:t}, Y_{i,t}}  \frac{\Pro{A_{i,t-k+1:t}}}{\Pro{A_{i,t-k+1:t} | L_{i,1:t}}} \times\\
	&\left(Y_{i,t}(A_{i,t-k+1:t}, L_{i,1:t}) - \T(i,t,A_{i,t-k+1:t})\right)^2 \\
	&= \sum_{A_{i,t-k+1:t}} \sum_{L_{i,1:t}} \sum_{Y_{i,t}} \Pro{A_{i,t-k+1:t}} \Pro{L_{i,1:t}} \Pro{Y_{i,t}| A_{i,t-k+1:t}, L_{i,1:t}} \times \\
	&\left(Y_{i,t}(A_{i,t-k+1:t}, L_{i,1:t}) - \T(i,t,A_{i,t-k+1:t})\right)^2 \\
	&= \sum_{A_{i,t-k+1:t}} \Pro{A_{i,t-k+1:t}}^2 \sum_{L_{i,1:t}}  \Pro{L_{i,1:t}| A_{i,t-k+1:t}}\sum_{Y_{i,t}}  \Pro{Y_{i,t}| A_{i,t-k+1:t}, L_{i,1:t}} \times \\
&\left(Y_{i,t}(A_{i,t-k+1:t}, L_{i,1:t}) - \T(i,t,A_{i,t-k+1:t})\right)^2 \\
	\end{align*}
\end{proof}

	\begin{lemma}\label{lem:wt-bound}
	$w_{i,t} \leq \left(\frac{1-\delta}{\delta} \right)^k$
\end{lemma}
\begin{proof}
	\begin{align*}
	w_{i,t} = \prod_{s=t-k+1}^t \frac{\Pro{A_{i,s} | A_{i,t-k+1:s-1}}}{\Pro{A_{i,s} | A_{i,t-k+1:s-1}, L_{i,1:s}} }
	\end{align*}
	Recall that the given policy satisfies positivity with constant $\delta$ i.e. for each $A_{i,1:s-1}, L_{i,1:s}, Y_{i,1:s-1}$, we have
	\[
	\delta < \Pro{A_{i,s} | A_{i,1:s-1}, L_{i,1:s}} < 1 - \delta
	\]
	First, consider the term in the denominator.
	\begin{align*}
	\Pro{A_{i,s} | A_{i,t-k+1:s-1}, L_{i,1:s}}  &= \sum_{A_{i,1:t-k}} \Pro{A_{i,1:t-k}| L_{i,1:s-1}} \Pro{A_{i,s} | A_{i,1:s-1}, L_{i,1:s}} \\
	&\geq \delta \sum_{A_{i,1:t-k}}\Pro{A_{i,1:t-k}| L_{i,1:s-1}} = \delta
	\end{align*}
	This gives a lower  bound of $\delta$ on the term in the denominator. Now consider the term in the numerator.
	Positivity implies that $\Pro{A_{i,s} | A_{i,1:s-1}} = \Pro{A_{i,s}| A_{i,1:s-1}, L_{i,1:s}, Y_{i,1:s-1}} \Pro{L_{i,1:s}, Y_{i,1:s-1}} \leq 1 - \delta$. 
	These two results imply that each term in the product of $w_{i,t}$ 
	is bounded by $(1-\delta)/\delta$ and we get the desired bound on $w_{i,t}$. 
\end{proof}

\begin{lemma}\label{lem:wtwptlogit}
	\begin{align*}
	&\Em{H_{i,1:t'}}{w_{i,t} w_{i,t'} (Y_{i,t} - \T(i,t,S(i,t))^2) (Y_{i,t'} - \T(i,t',S(i,t'))^2)} \\
	&= \sum_{A_{i,t-k+1:t}} \sum_{L_{i,1:t}} \sum_{Y_{i,t}} \Pro{A_{i,t-k+1:t}} \Pro{L_{i,1:t}} \Pro{Y_{i,t}|A_{i,t-k+1:t},L_{i,1:t}}\times  \\
	&\left(Y_{i,t}(A_{i,t-k+1:t}, L_{i,1:t}) - \T(i,t,A_{i,t-k+1:t})\right)^2\times \\
	& \sum_{A_{i,t'-k+1:t'}} \sum_{L_{i,t+1:t'}} \sum_{Y_{i,t'}} \Pro{A_{i,t'-k+1:t'}} \Pro{A_{i,t'-k+1:t'}|L_{i,1:t},A_{i,t-k+1:t}}\times \\ &\Pro{L_{i,t+1:t'}|A_{i,t-k+1:t}, L_{i,1:t},A_{i,t'-k+1:t'}} \times 
	\\& \Pro{Y_{i,t'}|A_{i,t'-k+1:t'},L_{i,1:t'}} \left(Y_{i,t'}(A_{i,t'-k+1:t'}, L_{i,1:t'}) - \T(i,t',A_{i,t'-k+1:t'})\right)^2  \\
	\end{align*} 
\end{lemma}
\begin{proof}
	We assume that the outcome variable is discrete. The proof for continuous variable is similar.
	\begin{align*}
	&\Em{H_{i,1:t'}}{w_{i,t} w_{i,t'} (Y_{i,t} - \T(i,t,S(i,t))^2) (Y_{i,t'} - \T(i,t',S(i,t'))^2)} \\ 
	&= \sum_{A_{i,1:t'}} \sum_{L_{i,1:t'}} \sum_{Y_{i,1:t'}} \Pro{A_{i,1:t'}, L_{i,1:t'}, Y_{i,1:t'}} \times \\  &\prod_{s=t-k+1}^t \frac{\Pro{A_{i,s} | A_{i,t-k+1:s-1}}}{\Pro{A_{i,s} | A_{i,t-k+1:s-1}, L_{i,1:s} } } \times \prod_{s=t-k+1}^t \frac{\Pro{A_{i,s} | A_{i,t-k+1:s-1}}}{\Pro{A_{i,s} | A_{i,t-k+1:s-1}, L_{i,1:s} } } \times  \\
	&\left(Y_{i,t}(A_{i,t-k+1:t}, L_{i,1:t}) - \T(i,t,A_{i,t-k+1:t})\right)^2 \times \left(Y_{i,t'}(A_{i,t'-k+1:t'}, L_{i,1:t'}) - \T(i,t',A_{i,t'-k+1:t'})\right)^2 \\
	&= \sum_{A_{i,1:t'}} \sum_{L_{i,1:t'}} \sum_{Y_{i,1:t'}}\Pro{A_{i,1:t'}, L_{i,1:t'}, Y_{i,1:t'}}   \frac{\Pro{A_{i,t-k+1:t}}}{\Pro{A_{i,t-k+1:t} | L_{i,1:t}}} \times \frac{\Pro{A_{i,t'-k+1:t'}}}{\Pro{A_{i,t'-k+1:t'} | L_{i,1:t'}}} \times\\
	&\left(Y_{i,t}(A_{i,t-k+1:t}, L_{i,1:t}) - \T(i,t,A_{i,t-k+1:t})\right)^2 \times \left(Y_{i,t'}(A_{i,t'-k+1:t'}, L_{i,1:t'}) - \T(i,t',A_{i,t'-k+1:t'})\right)^2 \\
	&= \sum_{A_{i,t-k+1:t}} \sum_{A_{i,t'-k+1:t'}} \sum_{L_{i,1:t'}} \sum_{Y_{i,t}} \sum_{Y_{i,t'}} \Pro{A_{i,t-k+1:t}, A_{i,t'-k+1:t'}, L_{i,1:t'}, Y_{i,t}, Y_{i,t'}} \times \\ &\frac{\Pro{A_{i,t-k+1:t}}}{\Pro{A_{i,t-k+1:t} | L_{i,1:t}}} \times \frac{\Pro{A_{i,t'-k+1:t'}}}{\Pro{A_{i,t'-k+1:t'} | L_{i,1:t'}}} \times\\
	&\left(Y_{i,t}(A_{i,t-k+1:t}, L_{i,1:t}) - \T(i,t,A_{i,t-k+1:t})\right)^2 \times \left(Y_{i,t'}(A_{i,t'-k+1:t'}, L_{i,1:t'}) - \T(i,t',A_{i,t'-k+1:t'})\right)^2  \\
	&= \sum_{A_{i,t-k+1:t}} \sum_{L_{i,1:t}} \sum_{Y_{i,t}} \Pro{A_{i,t-k+1:t}} \Pro{L_{i,1:t}} \Pro{Y_{i,t}|A_{i,t-k+1:t},L_{i,1:t}}\times  \\
	&\left(Y_{i,t}(A_{i,t-k+1:t}, L_{i,1:t}) - \T(i,t,A_{i,t-k+1:t})\right)^2\times \\
	& \sum_{A_{i,t'-k+1:t'}} \sum_{L_{i,t+1:t'}} \sum_{Y_{i,t'}} \Pro{A_{i,t'-k+1:t'}|L_{i,1:t},A_{i,t-k+1:t}} \Pro{L_{i,t+1:t'}|A_{i,t-k+1:t}, L_{i,1:t},A_{i,t'-k+1:t'}} \times \\ &\frac{ \Pro{L_{i,1:t'}} }{\Pro{L_{i,1:t'} | A_{i,t'-k+1:t'}}} \times 
	\\& \Pro{Y_{i,t'}|A_{i,t'-k+1:t'},L_{i,1:t'}} \left(Y_{i,t'}(A_{i,t'-k+1:t'}, L_{i,1:t'}) - \T(i,t',A_{i,t'-k+1:t'})\right)^2  \\
	&= \sum_{A_{i,t-k+1:t}} \sum_{L_{i,1:t}} \sum_{Y_{i,t}} \Pro{A_{i,t-k+1:t}} \Pro{L_{i,1:t}} \Pro{Y_{i,t}|A_{i,t-k+1:t},L_{i,1:t}}\times  \\
	&\left(Y_{i,t}(A_{i,t-k+1:t}, L_{i,1:t}) - \T(i,t,A_{i,t-k+1:t})\right)^2\times \\
	& \sum_{A_{i,t'-k+1:t'}} \sum_{L_{i,t+1:t'}} \sum_{Y_{i,t'}} \Pro{A_{i,t'-k+1:t'}} \Pro{A_{i,t'-k+1:t'}|L_{i,1:t},A_{i,t-k+1:t}} \times \\ &\Pro{L_{i,t+1:t'}|A_{i,t-k+1:t}, L_{i,1:t},A_{i,t'-k+1:t'}} \times 
	\\& \Pro{Y_{i,t'}|A_{i,t'-k+1:t'},L_{i,1:t'}} \left(Y_{i,t'}(A_{i,t'-k+1:t'}, L_{i,1:t'}) - \T(i,t',A_{i,t'-k+1:t'})\right)^2  \\
	\end{align*}
\end{proof}
	\begin{lemma}\label{lem:concave-likelihood}
	$L_{N,T}(\T)$ is convex in $\T$.
\end{lemma}
\begin{proof}
	\begin{align*}
	L_{N,T}(\T) &= \frac{1}{NT} \sum_{i=1}^N \sum_{t=1}^T w_{i,t} (Y_{i,t} - \T(i,t,S(i,t)))^2 
	\end{align*}
	Each term inside the summation i.e. $(Y_{i,t} - \T_{N,T}(i,t,S(i,t)))^2$ is a convex function. The likelihood
	function is a non-negative weighted sum of convex functions and is also convex.
\end{proof}

\begin{lemma}\label{lem:gap-likelihood}
	Let $\mathcal{N}$ be a $d$-neighborhood of $\T^\star_{N,T}$ i.e. $\mathcal{N} = \set{\T : \norm{\T^\star_{N,T} - \T}_2/\sqrt{NT} \leq d}$. Then
	for any $\T' \notin \mathcal{N}$ we have $\ell^*_{N,T}(\T^\star_{N,T}) < \ell^*_{N,T}(\T') - \delta^{2k}  d^2$.
\end{lemma}
\begin{proof}
	Lemma \ref{lem:wtlogit} gives us
	\begin{align*}
	&\ell^*_{N,T}(\T^\star_{N,T}) - \ell^*_{N,T}(\T')  \\ &= \frac{1}{NT} \sum_{i,t} \sum_{A_{i,t-k+1:t}} \Pro{A_{i,t-k+1:t}}^2 \sum_{L_{i,1:t}}  \Pro{L_{i,1:t}| A_{i,t-k+1:t}}\sum_{Y_{i,t}}  \Pro{Y_{i,t}| A_{i,t-k+1:t}, L_{i,1:t}} \times \\
	&\left(Y_{i,t}(A_{i,t-k+1:t}, L_{i,1:t}) - \T^\star_{N,T}(i,t,A_{i,t-k+1:t})\right)^2 - \left(Y_{i,t}(A_{i,t-k+1:t}, L_{i,1:t}) - \T'(i,t,A_{i,t-k+1:t})\right)^2 \\
	&= \sum_{A_{i,t-k+1:t}} \Pro{A_{i,t-k+1:t}}^2 \sum_{L_{i,1:t}}  \Pro{L_{i,1:t}| A_{i,t-k+1:t}}\sum_{Y_{i,t}}  \Pro{Y_{i,t}| A_{i,t-k+1:t}, L_{i,1:t}} \times \\
	& \left\{ -2 Y_{i,t}(A_{i,t-k+1:t}, L_{i,1:t}) (\T^\star_{N,T}(i,t,A_{i,t-k+1:t}) - \T'(i,t,A_{i,t-k+1:t})) \right. \\  &+ \left. (\T^\star_{N,T}(i,t,A_{i,t-k+1:t}))^2 - (\T'(i,t,A_{i,t-k+1:t}))^2  \right\}
	\end{align*}
	Now we use the definition of $\T^\star_{N,T}$ i.e.
	$$\T^\star_{N,T}(i,t,A_{i,t-k+1:t}) = \Pro{L_{i,1:t}| A_{i,t-k+1:t}}\sum_{Y_{i,t}}  \Pro{Y_{i,t}| A_{i,t-k+1:t}, L_{i,1:t}} Y_{i,t}(A_{i,t-k+1:t}, L_{i,1:t})$$
	and get 
	\begin{align*}
	&\ell^*_{N,T}(\T^\star_{N,T}) - \ell^*_{N,T}(\T')  = \frac{1}{NT} \sum_{i,t}  \sum_{A_{i,t-k+1:t}} \Pro{A_{i,t-k+1:t}}^2 \times \\
	& \left\{ -2 \T^\star_{N,T}(i,t,A_{i,t-k+1:t}) (\T^\star_{N,T}(i,t,A_{i,t-k+1:t}) - \T'(i,t,A_{i,t-k+1:t})) \right. \\  &+ \left. (\T^\star_{N,T}(i,t,A_{i,t-k+1:t}))^2 - (\T'(i,t,A_{i,t-k+1:t}))^2  \right\} \\
	&= - \frac{1}{NT} \sum_{i,t}  \sum_{A_{i,t-k+1:t}} \Pro{A_{i,t-k+1:t}}^2 \left(\T^\star_{N,T}(i,t,A_{i,t-k+1:t}) -  \T'(i,t,A_{i,t-k+1:t})\right)^2 \\
	&\leq - \delta^{2k} \frac{1}{NT} \norm{\T^\star_{N,T} - \T'}_2^2 \leq - \delta^{2k} d^2
	\end{align*}
	The final step uses positivity and the fact that $\norm{\T^\star_{N,T} - \T'}_2/\sqrt{NT} > d$.
\end{proof}

	\begin{lemma}\label{lem:hoeffding}
		Suppose $\norm{\T}_{\infty} \le L$. Then
	$$\Pro{\abs{\frac{1}{NT}\sum_{p=1}^{2^k} \sum_{(i,t) \in S_p} w_{i,t} \T^2(i,t,p) - \frac{1}{NT}\norm{\T}_W^2} \geq \eps} \leq 2 \exp\left( - \frac{2N \eps^2}{L^4\left(\frac{1-\delta}{\delta}\right)^{2k}}\right)$$
\end{lemma}
\begin{proof}
	Suppose $S(i,t)$ be the slice selected by the policy for agent $i$ at time $t$. Then we have $\frac{1}{NT}\sum_{p=1}^{2^k} \sum_{(i,t) \in S_p} w_{i,t} \T^2(i,t,p) = \frac{1}{NT} \sum_{i=1}^N \sum_{t=1}^T w_{i,t} \T^2(i,t,S(i,t))$.
	Observe that for each $i$, $$\frac{1}{T} \sum_{t=1}^T w_{i,t} \T^2(i,t,S(i,t)) \in \left[0,L^2\left(\frac{1-\delta}{\delta}\right)^k \right]$$. Now we apply the Hoeffding inequality considering the
	random variables $\left\{ \frac{1}{T} \sum_{t=1}^T w_{i,t} \T^2(i,t,S(i,t)) \right\}_{i=1}^N$ as independent random 
	variables and get the following bound.
	\begin{equation}
	\Pro{\abs{\frac{1}{NT}\sum_{p=1}^B \sum_{(i,t) \in S_p} w_{i,t} \T^2(i,t,p) - \frac{1}{NT}\norm{\T}_W^2} \geq \eps} \leq 2 \exp\left( - \frac{2N \eps^2}{L^4\left(\frac{1-\delta}{\delta}\right)^{2k}}\right)
	\end{equation}
\end{proof}

 \begin{lemma}\label{lem:chebyshev-time}
 Suppose $\norm{\T}_{\infty} \le L$ and \ref{asn:decay} holds. Then
                $$\Pro{\abs{\frac{1}{NT}\sum_{p=1}^B \sum_{(i,t) \in O_p} w_{i,t} \T^2(i,t,p) - \frac{1}{NT}\norm{\T}_W^2} \geq \eps} \leq  \bigo{\frac{ L^4 \left( \frac{1-\delta}{\delta}\right)^{2k}}{\eps^2 T^{\gamma}}}$$
        \end{lemma}
        \begin{proof}
                Suppose $S(i,t)$ be the slice selected by the policy for agent $i$ at time $t$. Then we have $\frac{1}{NT}\sum_{p=1}^B \sum_{(i,t) \in O_p} w_{i,t} \T^2(i,t,p) = \frac{1}{NT} \sum_{i=1}^N \sum_{t=1}^T w_{i,t} \T^2(i,t,S(i,t))$.
                Observe that for each $t$, $$\frac{1}{N} \sum_{i=1}^N w_{i,t} \T^2(i,t,S(i,t)) \in \left[0,L^2\left(\frac{1-\delta}{\delta}\right)^k \right]$$. Now we apply the Chebyshev inequality considering the
                random variables $\left\{ \frac{1}{N} \sum_{i=1}^N w_{i,t} \T^2(i,t,S(i,t)) \right\}_{t=1}^T$.
                \begin{align*}
                &\Pro{\abs{\frac{1}{NT}\sum_{p=1}^B \sum_{(i,t) \in O_p} w_{i,t} \T^2(i,t,p) - \frac{1}{NT}\norm{\T}_W^2} \geq \eps} \leq \frac{\var{\frac{1}{NT} 
                                \sum_{t=1}^T \sum_{i=1}^N w_{i,t} \T^2(i,t,S(i,t))}}{\eps^2} \\
                &= \frac{1}{\eps^2 T^2} \left\{ \sum_{t=1}^T \var{\frac{1}{N}\sum_{i=1}^N w_{i,t} \T^2(i,t,S(i,t))} +\right.\\& \left.2 \sum_{t' - t < k} \cov\left(\frac{1}{N}\sum_{i=1}^N w_{i,t} \T^2(i,t,S(i,t)), \frac{1}{N}\sum_{i=1}^N w_{i,t'} \T^2(i,t',S(i,t'))\right) \right.\\
                &\left. + 2 \sum_{t' - t \geq k} \cov\left(\frac{1}{N}\sum_{i=1}^N w_{i,t} \T^2(i,t,S(i,t)), \frac{1}{N}\sum_{i=1}^N w_{i,t'} \T^2(i,t',S(i,t')) \right)\right\}\\
                &\leq \frac{1}{\eps^2 T^2} \left\{ T L^4 \left( \frac{1-\delta}{\delta}\right)^{2k} + 2Tk L^4 \left( \frac{1-\delta}{\delta}\right)^{2k} + 2 c \sum_{t' - t \geq k} L^4 \left( \frac{1-\delta}{\delta}\right)^{2k} (t'-t)^{1-\gamma}\right\} \\
                &\leq \frac{1}{\eps^2 T^2} \left\{ T L^4 \left( \frac{1-\delta}{\delta}\right)^{2k} + 2Tk L^4 \left( \frac{1-\delta}{\delta}\right)^{2k} + 2 T^{2-\gamma} L^4 \left( \frac{1-\delta}{\delta}\right)^{2k}\right\} = \bigo{\frac{ L^4 \left( \frac{1-\delta}{\delta}\right)^{2k}}{\eps^2 T^{\gamma} }}
                \end{align*}
        \end{proof}

\begin{lemma}\label{lem:approx-optim}
\begin{itemize}
	\item If $k = \bigo{\log_{(1-\delta)/\delta} N}$, then  $\frac{1}{NT}\sum_{p = 1}^{2^k} \sum_{(i,t) \in S_p} w_{i,t} \left(Y_{i,t} - \checkT_{N,T}(i,t,p)\right)^2 \leq  \opt + \bigo{\frac{L^2}{N^{5/4}} }$ with  probability at least $1 - \exp\left( -N^{1/4}\right)$.
	\item If $k = \bigo{\log_{(1-\delta)/\delta} T}$ and \ref{asn:decay} holds, then $\frac{1}{NT}\sum_{p = 1}^{2^k} \sum_{(i,t) \in S_p} w_{i,t} \left(Y_{i,t} - \checkT_{N,T}(i,t,p)\right)^2 \leq  \opt + \bigo{\frac{L^2}{T^{7/8}} }$ with  probability at least $1 - T^{-\gamma}$.
\end{itemize}
\end{lemma}
\begin{proof}
	Lemma \ref{lem:hoeffding} proves the following result: $$\Pro{ \sum_{p = 1}^{2^k} \sum_{(i,t) \in S_p} w_{i,t} \T(i,t,p)^2 \notin [\norm{\T}^2_W - \eps, \norm{\T}^2_W + \eps]} \leq \bigo{\exp\left( - \frac{2N \eps^2}{L^4\left(\frac{1-\delta}{\delta}\right)^{2k}}\right)}. $$
	Suppose ${\hatT}_{N,T}$ solves \ref{eq:wt-least-square} and $\checkT_{N,T}$ solves \ref{eq:new-objective}, then we get the following bound with probability at least $1 - \exp\left( - \frac{2N \eps^2}{L^4\left(\frac{1-\delta}{\delta}\right)^{2k}}\right)$: 
	\begin{align*}
	&\frac{1}{NT}\sum_{p = 1}^{2^k} \sum_{(i,t) \in S_p} w_{i,t} \left(Y_{i,t} - \checkT_{N,T}(i,t,p)\right)^2 \\
	&= \frac{1}{NT}\sum_{p=1}^{2^k} \sum_{(i,t) \in S_p} w_{i,t} Y^2_{i,t} - \frac{2}{NT} \sum_{p=1}^{2^k} \sum_{(i,t) \in S_p} w_{i,t} Y_{i,t} \checkT_{N,T}(i,t,p) + 
	\frac{1}{NT} \sum_{p=1}^{2^k} \sum_{(i,t) \in S_p} w_{i,t} \left( \checkT_{N,T}(i,t,p)\right)^2\\
	&\leq \frac{1}{NT} \left(\norm{Y_w}_W^2 + \eps \right)  - 2 \sum_{i,t,p} W(i,t,p) Y_w(i,t,p) \checkT_{N,T}(i,t,p) + \frac{1}{NT} \left(\norm{\checkT_{N,T}}^2_W + \eps \right) \\
	&= \frac{1}{NT} \norm{Y_w - \checkT_{N,T}}^2_W + \frac{2\eps}{NT} \\
	&\leq \frac{1}{NT} \norm{Y_w - {\hatT}_{N,T}}^2_W +  \frac{2\eps}{NT} \\
	&= \frac{1}{NT} \left[ \norm{Y_w}_W^2 - 2 \sum_{i,t,p} W(i,t,p) Y_w(i,t,p) \hatT_{N,T}(i,t,p) + \norm{\hatT_{N,T}}_W^2\right] + \frac{2\eps}{NT} \\
	&\leq \frac{1}{NT} \sum_{p = 1}\sum_{(i,t) \in S_p} w(i,t)\left(Y_{i,t} - \hatT_{N,T}(i,t,p)\right)^2 + \frac{4\eps}{NT} \\
	\end{align*}
	The first and the third inequality use lemma \ref{lem:hoeffding} and the second inequality uses the fact that $\checkT_{N,T}$ is the optimal solution to \ref{eq:new-objective}. Now if we
	substitute $k = 1/8 \log_{(1-\delta)/\delta} N$ and $\eps = \bigo{\frac{L^2}{N^{1/4}} }$, we get the first result.

	We now consider the case when $N$ is fixed and $T$ increases to infinity. Suppose \ref{asn:decay} holds and $k = 1/8 \log_{(1-\delta)/\delta} T$. Then lemma \ref{lem:chebyshev-time} gives 	$\Pro{\abs{\frac{1}{NT}\sum_{p=1}^B \sum_{(i,t) \in O_p} w_{i,t} \T^2(i,t,p) - \frac{1}{NT}\norm{\T}_W^2} \geq \eps} \leq  \bigo{\frac{ L^4 \left( \frac{1-\delta}{\delta}\right)^{2k}}{\eps^2 T^{\gamma}}}$. We now proceed as before but substitute $\eps = L^2 T^{1/8}$ in the end.
\end{proof}
	\begin{lemma}\label{lem:covariance-bound}
		Suppose $t' > t + k$ and assumption \ref{asn:decay} holds. Then the following is true.
		\[
		 \cov\left(w_{i,t}\log\Pro{Y_{i,t}|\T_{N,T}}, w_{i,t'}\log\Pro{Y_{i,t'}|\T_{N,T}}\right)  \leq \bigo{ (t'-t)^{1-\gamma} }
		\]
	\end{lemma}
	\begin{proof}
		Let us write $\Hy_{i,1:t} = (a_{i,1:t},\ell_{i,1:t},y_{i,1:t-1})$ to denote the history upto time $t$ excluding the outcome at time $t$.
		\begin{align*}
		& \cov\left(w_{i,t}(Y_{it} - T(i,t,S(i,t)) )^2, w_{i,t'}(Y_{it'} - T(i,t',S(i,t')) )^2\right) \\&= \E{w_{i,t}(Y_{it} - T(i,t,S(i,t)) )^2 w_{i,t'}(Y_{it'} - T(i,t',S(i,t')) )^2 } \\&- \E{w_{i,t} (Y_{it} - T(i,t,S(i,t)) )^2}\E{w_{i,t'} (Y_{it'} - T(i,t',S(i,t')) )^2 } 
		\end{align*}
		Now we use lemma \ref{lem:wtlogit} and \ref{lem:wtwptlogit} to substitute the terms and obtain the following bound.
		\begin{align*}
		&\sum_{A_{i,t-k+1,t}}\sum_{L_{i,1:t}} \sum_{Y_{i,t}}  \Pro{A_{i,t-k+1:t}} \Pro{L_{i,1:t}} \Pro{Y_{i,t}|A_{i,t-k+1:t},L_{i,1:t}}\times  \\
	&\left(Y_{i,t}(A_{i,t-k+1:t}, L_{i,1:t}) - \T(i,t,A_{i,t-k+1:t})\right)^2\times   \\
	&\sum_{\tA_{i,t'-k+1:t'} } \sum_{\tL_{i,1:t'}} \sum_{\tY_{i,t'}} \Pro{\tilde{A}_{i,t'-k+1:t'}} \Pro{\tL_{i,1:t}} \Pro{\tY_{i,t'}|\tA_{i,t'-k+1:t'},\tL_{i,1:t'}}\\ &\left(\tY_{i,t'}(\tA_{i,t'-k+1:t'}, \tL_{i,1:t'}) - \T(i,t',\tA_{i,t'-k+1:t'})\right)^2 \\
	& \left(\Pro{\tA_{i,t'-k+1:t'}|\tL_{i,1:t},\tA_{i,t-k+1:t}} \Pro{\tL_{i,t+1:t'}|\tA_{i,t-k+1:t}, \tL_{i,1:t},\tA_{i,t'-k+1:t'}} \right. \\
	&- \left. \Pro{\tA_{i,t'-k+1:t'}} \Pro{\tL_{t+1:t'}|\tL_{i,1:t},\tA_{i,t'-k+1:t'}}  \right)  \\
	&= \sum_{A_{i,t-k+1,t}}\sum_{L_{i,1:t}} \sum_{Y_{i,t}}  \Pro{A_{i,t-k+1:t}} \Pro{L_{i,1:t}} \Pro{Y_{i,t}|A_{i,t-k+1:t},L_{i,1:t}} \times \\ &\left(Y_{i,t}(A_{i,t-k+1:t}, L_{i,1:t}) - \T(i,t,A_{i,t-k+1:t})\right)^2 \times \\
	&\sum_{\tA_{i,t'-k+1:t'} }\Pro{\tilde{A}_{i,t'-k+1:t'}}  \sum_{\tL_{i,1:t}} \Pro{\tL_{i,1:t}}  \sum_{\tL_{i,t+1:t'}} \E{(\tY_{i,t'} - \T(i,t',\tA_{i,t'-k+1:t'}))^2 | \tL_{i,1:t'}} \times \\
	&\left(\Pro{\tA_{i,t'-k+1:t'}|\tL_{i,1:t},\tA_{i,t-k+1:t}} \Pro{\tL_{i,t+1:t'}|\tA_{i,t-k+1:t}, \tL_{i,1:t},\tA_{i,t'-k+1:t'}} \right. \\
	&- \left. \Pro{\tA_{i,t'-k+1:t'}} \Pro{\tL_{t+1:t'}|\tL_{i,1:t},\tA_{i,t'-k+1:t'}}  \right) 
		\end{align*}
		Now we use assumption \ref{asn:decay} to bound the term on the last line. Let $\eps = (t' - t)^{1-\gamma}$.
		\begin{align*}
		&\Pro{\tA_{i,t'-k+1:t'}|\tL_{i,1:t},\tA_{i,t-k+1:t}} \Pro{\tL_{i,t+1:t'}|\tA_{i,t-k+1:t}, \tL_{i,1:t},\tA_{i,t'-k+1:t'}} \\
		&\le (1+\eps) \Pro{\tA_{i,t'-k+1:t'}} \frac{\Pro{\tL_{i,t+1:t'},\tA_{i,t'-k+1:t'}| \tL_{i,1:t},\tA_{i,t-k+1:t}}}{\Pro{\tA_{i,t'-k+1:t'}| \tL_{i,1:t},\tA_{i,t-k+1:t}}} \\
		&\le \frac{1 + \eps}{1 - \eps} \Pro{\tL_{i,t+1:t'},\tA_{i,t'-k+1:t'}| \tL_{i,1:t},\tA_{i,t-k+1:t}} \\
		&\le \frac{1 + \eps}{1 - \eps}  \frac{\Pro{\tL_{i,t+1:t'},\tA_{i,t-k+1:t}, \tL_{i,1:t},\tA_{i,t'-k+1:t'}} }{\Pro{\tL_{i,1:t},\tA_{i,t-k+1:t}} } \\
		&\le \frac{1 + \eps}{1 - \eps} \frac{\Pro{\tA_{i,t-k+1:t} |  \tL_{i,1:t}} \Pro{\tL_{i,1:t}} \Pro{\tL_{i,t+1:t'}|\tL_{i,1:t},\tA_{i,t'-k+1:t'}} \Pro{\tL_{i,1:t},\tA_{i,t'-k+1:t'}} }{\Pro{\tA_{i,t-k+1:t} | \tL_{i,1:t}} \Pro{\tL_{i,1:t}} }  \\
		&\le \frac{1 + \eps}{1 - \eps} \Pro{\tL_{i,t+1:t'}|\tL_{i,1:t},\tA_{i,t'-k+1:t'}} \Pro{\tL_{i,1:t},\tA_{i,t'-k+1:t'}} \\
		&\le (1 + 3\eps)\Pro{\tL_{i,t+1:t'}|\tL_{i,1:t},\tA_{i,t'-k+1:t'}} \Pro{\tA_{i,t'-k+1:t'}} 
		\end{align*}
		Substituting the above result, we get the following bound on the covariance.
		\begin{align*}
		&3\eps \sum_{A_{i,t-k+1,t}}\sum_{L_{i,1:t}} \sum_{Y_{i,t}}  \Pro{A_{i,t-k+1:t}} \Pro{L_{i,1:t}} \Pro{Y_{i,t}|A_{i,t-k+1:t},L_{i,1:t}} \times \\ &\left(Y_{i,t}(A_{i,t-k+1:t}, L_{i,1:t}) - \T(i,t,A_{i,t-k+1:t})\right)^2 \times \\
	&\sum_{\tA_{i,t'-k+1:t'} }\Pro{\tilde{A}_{i,t'-k+1:t'}}  \sum_{\tL_{i,1:t}} \Pro{\tL_{i,1:t}}  \sum_{\tL_{i,t+1:t'}} \E{(\tY_{i,t'} - \T(i,t',\tA_{i,t'-k+1:t'}))^2 | \tL_{i,1:t'}} \times \\
	& \Pro{\tA_{i,t'-k+1:t'}} \Pro{\tL_{t+1:t'}|\tL_{i,1:t},\tA_{i,t'-k+1:t'}} 
		\end{align*}
		Recall that the second moments of the couterfactuals are bounded. This implies that there exists a constant $M_2 > 0$ such that for all $\tA_{i,t'-k+1,t'}$ and $\tL_{i,1:t'}$, we have $\E{(\tY_{i,t'} - \T(i,t',\tA_{i,t'-k+1:t'}))^2 | \tL_{i,1:t'}}  < M_2$. Substituting this bound above and simplifying we get the following bound on the covariance.
		\begin{align*}
		&3M_2 \eps \sum_{A_{i,t-k+1,t}}\sum_{L_{i,1:t}} \sum_{Y_{i,t}}  \Pro{A_{i,t-k+1:t}} \Pro{L_{i,1:t}} \Pro{Y_{i,t}|A_{i,t-k+1:t},L_{i,1:t}} \times\\  &\left(Y_{i,t}(A_{i,t-k+1:t}, L_{i,1:t}) - \T(i,t,A_{i,t-k+1:t})\right)^2 
	\sum_{\tA_{i,t'-k+1:t'} }\Pro{\tilde{A}_{i,t'-k+1:t'}}^2  \sum_{\tL_{i,1:t}} \Pro{\tL_{i,1:t}}  \\
	&\le 3 M_2 \eps \sum_{A_{i,t-k+1,t}}\sum_{L_{i,1:t}}  \Pro{A_{i,t-k+1:t}} \Pro{L_{i,1:t}} \E{(Y_{i,t} - \T(i,t,A_{i,t-k+1:t}))^2| L_{i,1:t}} \\
	&\le 3 M_2^2 \eps
		\end{align*}
	\end{proof}
	
\end{document}